\documentclass{article}

\usepackage{microtype}
\usepackage{graphicx}
\usepackage{booktabs} 

\usepackage[bookmarksnumbered=true,bookmarksopen=true,bookmarksopenlevel=1]{hyperref}

\usepackage[accepted]{icml2024}

\usepackage{amsmath}
\usepackage{amssymb}
\usepackage{mathtools}
\usepackage{amsthm}

\usepackage[capitalize,noabbrev]{cleveref}

\theoremstyle{plain}
\newtheorem{theorem}{Theorem}[section]
\newtheorem{proposition}[theorem]{Proposition}

\theoremstyle{definition}

\theoremstyle{remark}
\newtheorem{remark}[theorem]{Remark}
\newtheorem{claim}[theorem]{Claim}

\usepackage{color}
\usepackage{bm}
\usepackage{bbm}
\usepackage{caption}
\usepackage{subcaption}
\usepackage{multirow}
\usepackage{siunitx}
\usepackage{xspace}

\usepackage{enumitem}
\setlist[itemize]{noitemsep, topsep=0.05em}
\setlist[enumerate]{noitemsep, topsep=0.05em}
\definecolor{todo}{RGB}{0,200,200}

\let\hat\widehat

\allowdisplaybreaks

\newcommand{\E}{{\mathbb{E}}}

\newcommand{\var}{\text{var}}
\newcommand{\cov}{\text{cov}}

\newcommand{\bg}{\bm{g}}

\newcommand{\bo}{\bm{o}}

\newcommand{\bx}{\bm{x}}

\newcommand{\bz}{\bm{z}}

\newcommand{\btheta}{\bm{\theta}}
\newcommand{\bphi}{\bm{\phi}}

\newcommand{\Lcal}{{\mathcal{L}}}

\newcommand{\loss}{\Lcal}
\newcommand{\lossaux}{\loss^{\mathsf{aux}}}
\newcommand{\astar}{a^{\star}}
\newcommand{\bstar}{b^{\star}}
\newcommand{\ahat}{\hat{a}}
\newcommand{\bhat}{\hat{b}}
\newcommand{\ehat}{\hat{e}}
\newcommand{\ehatp}{\hat{e}_{+}}

\newcommand{\sys}{\texttt{EE-LLM}\xspace}

\newcommand{\IN}{\mathsf{IN}}
\newcommand{\FE}{\mathsf{FE}}
\newcommand{\EE}{\mathsf{EE}}
\newcommand{\BB}{\mathsf{BB}}

\newcommand{\mact}{m^{\dagger}}
\newcommand{\osf}{\mathsf{o}}
\newcommand{\ind}{\mathbbm{1}}

\icmltitlerunning{\sys: Large-Scale Training and Inference of Early-Exit Large Language Models with 3D Parallelism}

\begin{document}

\twocolumn[
\icmltitle{\sys: Large-Scale Training and Inference of Early-Exit\\Large Language Models with 3D Parallelism}



\icmlsetsymbol{equal}{*}

\begin{icmlauthorlist}
\icmlauthor{Yanxi Chen}{equal,AlibabaCN}
\icmlauthor{Xuchen Pan}{equal,AlibabaCN}
\icmlauthor{Yaliang Li}{AlibabaCN}
\icmlauthor{Bolin Ding}{AlibabaCN}
\icmlauthor{Jingren Zhou}{AlibabaCN}
\end{icmlauthorlist}

\icmlaffiliation{AlibabaCN}{Alibaba Group. \{chenyanxi.cyx, panxuchen.pxc, yaliang.li, bolin.ding, jingren.zhou\}@alibaba-inc.com}

\icmlcorrespondingauthor{Yaliang Li}{yaliang.li@alibaba-inc.com}

\icmlkeywords{Large Language Models, Early Exiting, 3D Parallelism}

\vskip 0.3in
]



\printAffiliationsAndNotice{\icmlEqualContribution} 

\begin{abstract}

We present \sys, a framework for large-scale training and inference of early-exit large language models (LLMs).
While recent works have shown preliminary evidence for the efficacy of early exiting in accelerating LLM inference,
\sys makes a foundational step towards scaling up early-exit LLMs by supporting their training and inference with massive 3D parallelism. 
Built upon Megatron-LM, \sys implements a variety of algorithmic innovations and performance optimizations tailored to early exiting,
including a lightweight method that facilitates backpropagation for the early-exit training objective with pipeline parallelism,
techniques of leveraging idle resources in the original pipeline schedule for computation related to early-exit layers,
and two approaches of early-exit inference that are compatible with KV caching for autoregressive generation.
Our analytical and empirical study shows that \sys achieves great training efficiency
with negligible computational overhead compared to standard LLM training,
as well as outstanding inference speedup without compromising output quality.
To facilitate further research and adoption, 
we release \sys at \url{https://github.com/pan-x-c/EE-LLM}.
    
\end{abstract}

\section{Introduction}
\label{sec:introduction}

Large language models (LLMs) have amazed the world with their astonishing abilities and performance in solving a wide range of problems \cite{Brown2020,OpenAI2023,Chowdhery2022,Zhang2022,Touvron2023llama,Touvron2023llama2}. 
This is accompanied by excessive costs and carbon emissions for training and deploying these models, as their sizes have grown rapidly in recent years.
In general, costs for inference are dominant in the long run, as each model will be deployed to solve many problems for a long period of time.
This has inspired researchers and engineers to develop various approaches for accelerating LLM inference.

The focus of this work is \emph{early exiting}, which accelerates inference by allowing a deep neural network to make predictions and exit early in the network for certain inputs, without running the forward pass through the full network.
This is achieved by augmenting a standard neural network architecture (with a single exit at the end) with additional early-exit layers that transform intermediate hidden states into early outputs.
The early-exit model architecture, as visualized in Figure~\ref{fig:multi_exit_model}, not only retains the full capacity of a large model, but is also capable of adaptively using a smaller amount of computation for solving simpler problems.
The idea of early exiting is a natural analogue of how humans speak, think, and make decisions: 
not every problem requires or deserves the same amount of consideration, and one shall opt for fast reaction to simple problems without overthinking \cite{Kaya2018}.
Early exiting has been an active research area and widely applied in natural language processing, computer vision, and other areas \cite{Liu2020,Elbayad2020,Schwartz2020,Xin2020,Schuster2021,Teerapittayanon2016,Huang2018,Scardapane2020WhySW,Han2021DynamicNN,Xu2023SurveyDynamic}.
More recently, it has started to gain attention in the generative LLM domain \cite{Schuster2021,DelCorro2023,Bae2023,Varshney2023AcceleratingLI}, and is recognized as a promising direction for further reducing the latency and costs of LLM inference \cite{pope2022efficiently}.

\paragraph{Goal and motivations.}

Our primary goal is to build the infrastructure for \emph{scaling up} training and inference of early-exit LLMs.
This is motivated by the observation that sizes of early-exit models in prior works are still relatively small.
While the largest early-exit LLM that we are aware of has 13 billion parameters \cite{Varshney2023AcceleratingLI},
standard LLMs at much larger scales, e.g.~the 175B GPT-3 \cite{Brown2020}, 530B Megatron-Turing NLG \cite{Smith2022}, 540B PaLM \cite{Chowdhery2022}, or even larger sparsely activated models, have been well trained and deployed in many applications.
It is an urgent need for the community to truly understand the efficacy of early exiting for LLMs at larger scales,
which is indispensable for making early exiting a useful and practical option in complex scenarios that only sufficiently large LLMs can handle.

\paragraph{Challenges.}

The first and foremost question is how to train an early-exit LLM that is too large to fit into the memory of one single device (e.g.~GPU).
While state-of-the-art frameworks like 
Megatron-LM \cite{Shoeybi2019,Narayanan2021}, 
DeepSpeed \cite{Rasley2020DeepSpeedSO,Smith2022}, 
Alpa \cite{Zheng2022Alpa}, 
and many more, support training standard LLMs at large scales with data parallelism and model parallelism (including tensor, sequence, and pipeline parallelism), they do not provide native support for early-exit LLMs.
One particular challenge lies in pipeline parallelism \cite{Narayanan2019,Narayanan2021-PipeDream,Fan2021}, which partitions the model along the depth dimension into multiple pipeline stages, connected by limited point-to-point communication between devices; 
this seems to contradict training early-exit models, as the training objective is typically an aggregation of losses for multiple exits that are now located separately on different pipeline stages. 
Despite the necessity of pipeline parallelism in many scenarios, we are not aware of any implementation that supports training early-exit LLMs with it.

Moreover, training efficiency for early-exit generative LLMs requires special design,
since each early exit contains (at least) a large output embedding matrix that transforms hidden states into logits on the vocabulary,
which can constitute a non-trivial proportion of the whole model.
Naive implementation of early-exit LLM training can cause large computational overhead compared to standard LLM training.

Finally, with regard to autoregressive generation (where tokens are generated one by one, depending on previously generated tokens via the attention mechanism), 
naive implementation of early-exit inference is not compatible with KV caching, a standard technique of storing the keys and values of previously generated tokens at each layer.
Indeed, if the current token is generated via early exiting at some layer, then its KV caches in later layers are missing, which hinders the generation of future tokens.
Given that KV caching is enabled by default in most cases, the efficacy of early exiting for autoregressive generation might be questionable if its conflict with KV caching is not well resolved.

\paragraph{Main contributions.}

We propose \sys, a system for large-scale training and inference of {e}arly-{e}xit (EE) LLMs with 3D parallelism, which is designed to tackle the aforementioned challenges. 
\sys is built upon Megatron-LM \cite{Shoeybi2019,Narayanan2021,Smith2022}, and augments it with various functionalities for early exiting.
In addition to compatibility with existing functionalities of 3D parallelism provided by Megatron-LM, 
\sys also implements a variety of algorithmic innovations,
including a lightweight method that facilitates backpropagation for the early-exit training objective through pipeline stages,
various techniques of leveraging idle resources in the original pipeline schedule for computation related to early-exit layers,
and two approaches of early-exit inference that are compatible with KV caching. 
Implementation of \sys has been well optimized for maximum training and inference efficiency.
Our analytical and empirical study confirms that, 
with negligible computational overhead caused by early-exit layers during training with 3D parallelism, 
one obtains an early-exit LLM that generates tokens with adaptive token-wise exit selection, achieving outstanding inference speedup without compromising output quality.
With \sys, it is now possible to train and deploy early-exit LLMs that are as large as the maximum sizes of standard LLMs allowed by Megatron-LM, given the same amount of computational resources.
The source code for \sys can be found at \url{https://github.com/pan-x-c/EE-LLM}.

\begin{figure}[tbp]
\centering
\includegraphics[width=\columnwidth]{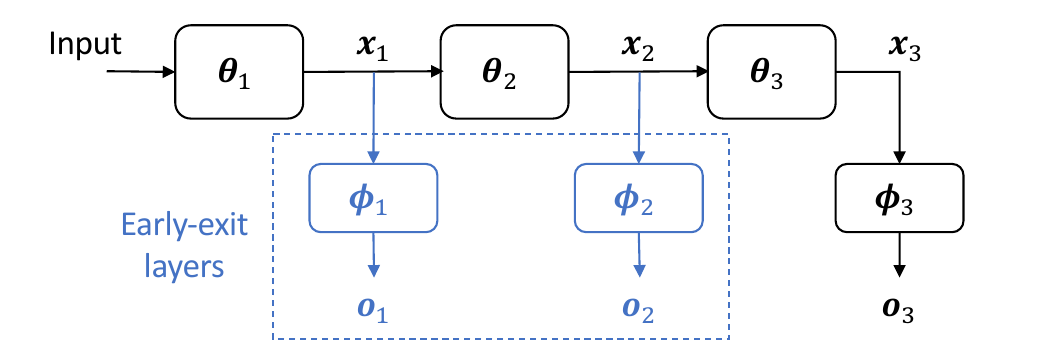}
\caption{The model architecture of an early-exit LLM.
Additional components compared to a standard LLM are highlighted in blue.
Each $\btheta_i$ represents a sequence of Transformer layers in the backbone of the LLM, with some additional modules in $\btheta_1$ for input processing.
Each $\bphi_i$ represents an early or final-exit layer that converts hidden states $\bx_i$ into output $\bo_i$, e.g.~logits for next-token prediction.
}
\label{fig:multi_exit_model}
\end{figure}

\paragraph{Organization.}
 
Section~\ref{sec:overview} provides a high-level overview of \sys,
while Sections~\ref{sec:training} and~\ref{sec:inference} focus on training and inference, respectively.
The efficacy of \sys is validated by numerical experiments in Section~\ref{sec:experiments}.
The appendix includes extended preliminaries and related works,
additional experiments,
as well as more details about the training efficiency, advanced features, and implementation of \sys.

\section{An overview of \sys}
\label{sec:overview}

This section provides an overview of our system for scaling up sizes, training and inference of early-exit LLMs, with flexible configurations and a wide range of functionalities.

\paragraph{Model architectures.}

We implement in \sys an early-exit Transformer architecture, which is built upon the generative pre-training (GPT) Transformer architecture \cite{Radford2018ImprovingLU,Radford2019LanguageMA} originally implemented in Megatron-LM. 
\sys allows users to
(1)~specify arbitrary layers to add early exits to;
(2)~add trainable modules to each early-exit layer, e.g.~a multi-layer perceptron (MLP) or a complete Transformer layer, 
on top of the \emph{minimalistic} structure (with an output embedding matrix, plus an optional layer normalization module in front of it);
and (3)~choose to tie \cite{Press2017,Schuster2022,Varshney2023AcceleratingLI} or untie the input and output embedding matrices of all early/final-exit layers.
Each option has its own pros and cons, as will be discussed in later sections.
With this in mind, \sys has been designed to cover a wide range of configurations, 
so that users can easily try them out and choose the most suitable ones for their own use cases.

\paragraph{Training.}

\sys contains the essential functionalities for training early-exit LLMs, which tackle the main challenges outlined in Section~\ref{sec:introduction}, i.e.~how to train with 3D parallelism while minimizing the computational overhead caused by early-exit layers.
In addition to substantial engineering efforts for compatibility with existing functionalities in Megatron-LM,
we design and implement a simple yet effective algorithm that facilitates pipeline parallelism with multiple early/final-exit training losses located on different pipeline stages, which is not possible in standard pipeline parallelism implemented by Megatron-LM or other frameworks.
Moreover, our analytical and empirical study shows that training an early-exit LLM with \sys is almost as efficient as training a standard LLM, in terms of training time and peak GPU memory. 
This is achieved by various performance optimizations that we design and implement in \sys, especially the ones that leverage idle computational resources in standard pipeline parallelism.
Finally, \sys contains some advanced features for fine-grained control or optimization of the training process, 
including the option of changing early-exit loss weights during training, 
and a novel method of further improving resource utilization by filling pipeline bubbles with useful computation.

\paragraph{Inference.}

We design and implement two methods to tackle the major challenge of early-exit LLM inference for autoregressive generation, 
namely the conflict between early exiting and KV caching (as explained in Section~\ref{sec:introduction}).
One method is based on KV recomputation, which runs the forward pass with a batch of recent tokens when generating each token.
%
The other method is based on a novel form of pipeline parallelism, which parallelizes the forward pass of the current token at a certain pipeline stage with some KV-related computation of previous tokens at later stages.

\section{Training}
\label{sec:training}

We first present in Section~\ref{subsec:bp_through_pipeline_stages} the essentials of scaling up early-exit LLM training with 3D parallelism, 
including a novel approach to executing backpropagation for the early-exit training objective through pipeline stages.
Then, we analyze the training efficiency achieved by \sys in Section~\ref{subsec:training_efficiency_brief}, 
and explore some advanced features in Section~\ref{subsec:additional_features}.

\subsection{Backpropagation through pipeline stages}
\label{subsec:bp_through_pipeline_stages}

The standard objective function for training an early-exit model is a weighted sum of losses at early and final exits. 
More formally, to train an early-exit LLM with $N$ exits (including the final output), we aim to solve
\begin{equation}
\min \quad \loss \coloneqq \sum_{i \in [N]} w_i \loss_i^{\mathsf{exit}},
\label{eq:multi_exit_training_loss}
\end{equation}
where $[N] = \{1, 2, \dots, N\}$, and each $\loss_i^{\mathsf{exit}}$ is a standard loss function for LLM pre-training (e.g.~negative log-likelihood of next-token prediction), calculated with outputs from the $i$-th exit.
The loss weights $\{w_i\}$ are hyperparameters specified by the user.

Our implementation for optimizing the loss function in Eq.~\eqref{eq:multi_exit_training_loss} is compatible with all types of parallelism in Megatron-LM.
Indeed, with some engineering efforts, existing functionalities in Megatron-LM for data and tensor/sequence parallelism are directly applicable.
The major challenge lies in {pipeline parallelism}, since it is not immediately clear how to calculate gradients for Eq.~\eqref{eq:multi_exit_training_loss} via backpropagation through pipeline stages.
In a single-GPU scenario with vanilla PyTorch, one simply needs to define \verb|loss| as the weighted sum of losses at all exits, and then run \verb|loss.backward()| for gradient calculation.
This is not feasible with pipeline parallelism, since losses $\{\loss_i^{\mathsf{exit}}\}$ are now located on different GPUs, and there is only limited P2P communication between each pair of adjacent stages.
On the other hand, Megatron-LM only supports backpropagation for a single loss function defined in the last stage.

\subsubsection{Methodology}
\label{subsubsec:Methodology}

To tackle this challenge, we propose a lightweight algorithm that instructs each pipeline stage to calculate the desired gradients correctly, without any additional communication overhead between stages.
To explain our method, let us first re-write the training objective defined in Eq.~\eqref{eq:multi_exit_training_loss} as 
$\loss = \sum_{i \in [K]} \loss_i$,
where $K$ represents the number of pipeline stages, and each $\loss_i$ is itself a weighted sum of one or multiple early/final-exit losses within Stage $i$.\footnote{Obviously, the objective in Eq.~\eqref{eq:multi_exit_training_loss} is a special case of this formulation. 
In fact, our analysis in the following allows each $\loss_i$ to be a general objective function defined locally in Stage $i$.
}
Consider one data sample $\bx$ for simplicity, and each loss function is calculated with $\bx$, i.e.~$\loss_i = \loss_i(\bx)$;
in addition, let $\bx_i$ be the hidden states that Stage $i$ calculates and sends to its next stage during the forward step.
Then, during the backward step, Stage $i$ receives some gradient tensor $\bg_i$ from Stage $i+1$, defines some \emph{auxiliary loss} $\lossaux_i$, and performs usual backward computation for $\lossaux_i$.
The auxiliary losses $\{\lossaux_i\}_{i \in [K]}$ and gradient tensors $\{\bg_i\}_{i \in [K-1]}$ are defined inductively:
\begin{subequations}
\label{eq:def_auxiliary_loss}
\begin{align}
\lossaux_{K} &\coloneqq \loss_K; \quad\text{for  } i = K-1, K-2, \dots, 1, \\
\lossaux_{i} &\coloneqq \loss_i + \langle \bg_i, \bx_i \rangle, 
\quad\text{where}\quad 
\bg_i \coloneqq \frac{\partial \lossaux_{i+1}}{\partial \bx_i}.
\end{align}
\end{subequations}
Intuitively, the linear term $\langle \bg_i, \bx_i \rangle$, i.e.~the sum of entrywise product between $\bg_i$ and $\bx_i$, summarizes information about the gradients of all losses located in later stages.
Note that $\bg_i$ is regarded by Stage $i$ as a constant tensor, and no gradient is calculated with respect to it.
A visualization of this process can be found in Figure~\ref{fig:backprop_pipeline}.
It has the same P2P communication scheme as in the case of training a standard LLM with pipeline parallelism;
the only difference is how each gradient tensors $\bg_i$ is defined locally in Stage $i+1$.
In the following, we prove that the proposed method leads to correct gradient calculation for the training objective $\loss$. 


\begin{figure}
\centering
\includegraphics[width=.95\columnwidth]{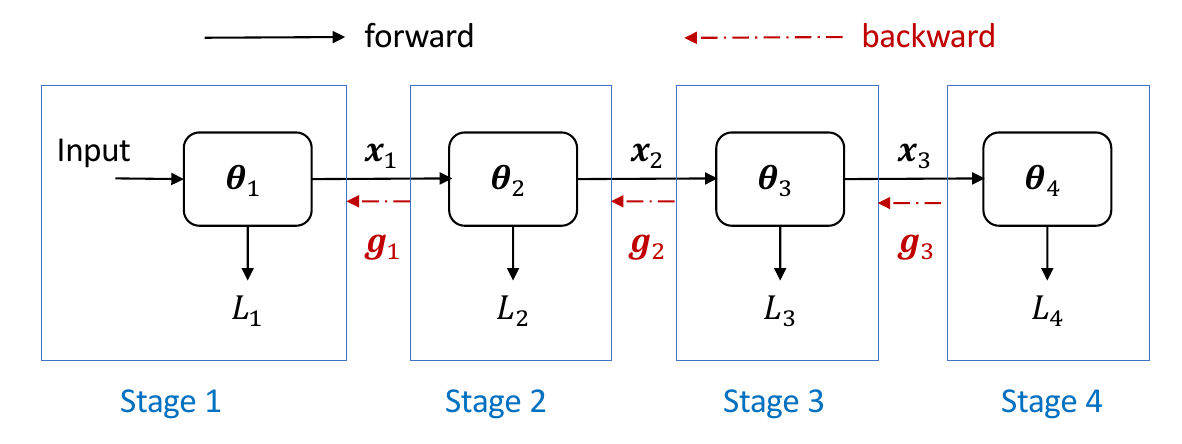}
\vspace{-0.5em}
\caption{The backpropagation process
for an early-exit model partitioned into four pipeline stages.}
\label{fig:backprop_pipeline}
\end{figure}

\subsubsection{Rationale}

Let us first prove the correctness of our solution under the assumption that there is no tied/shared parameter across pipeline stages, just for simplicity;
we will see very soon that this assumption is not essential and can be safely discarded.

\begin{proposition}
\label{prop:auxiliary_loss_bp}
Suppose that there is no tied parameter across pipeline stages, and consider the auxiliary losses defined in Eq.~\eqref{eq:def_auxiliary_loss}.
Then, for any $i \in [K]$ and any model parameter or activation tensor $\bz$ in Stage $i$, it holds that 
\begin{equation}
\label{eq:auxiliary_loss_prop}
\frac{\partial \lossaux_i}{\partial \bz} = \frac{\partial (\sum_{j=i}^{K} \loss_j)}{\partial \bz}.
\end{equation}
\end{proposition}

Notice that for any model parameter $\bz$ in Stage $i$, one has $\partial \loss_j / \partial \bz = \bm{0}$ for any $j < i$, due to the sequential structure of early-exit LLMs (or other deep neural networks).
Combining this with $\loss = \sum_{i \in [K]} \loss_i$ and Eq.~\eqref{eq:auxiliary_loss_prop} yields
\begin{equation*}
\frac{\partial \lossaux_i}{\partial \bz} = \frac{\partial (\sum_{j=i}^{K} \loss_j)}{\partial \bz} = \frac{\partial \loss}{\partial \bz},
\end{equation*}
implying correctness of gradient calculation for $\loss$. 

\begin{proof}[Proof of Proposition~\ref{prop:auxiliary_loss_bp}]
The claim of Eq.~\eqref{eq:auxiliary_loss_prop} is obviously true for the base case $i=K$, by definition of $\lossaux_K = \loss_K$.
Let us prove by induction for the remaining stages.
Suppose that Eq.~\eqref{eq:auxiliary_loss_prop} holds true for Stage $i+1$, namely 
${\partial \lossaux_{i+1}}/{\partial \bz} = {\partial (\sum_{j=i+1}^{K} \loss_j)}/{\partial \bz}$.
To prove Eq.~\eqref{eq:auxiliary_loss_prop} for Stage $i$, first note that by definition of $\bg_i$, we have
\begin{equation*}
\bg_i = \frac{\partial \lossaux_{i+1}}{\partial \bx_i} = \frac{\partial (\sum_{j=i+1}^{K} \loss_j)}{\partial \bx_i}.
\end{equation*}
Then, for any model parameter or activation tensor $\bz$ in Stage $i$, the following holds:
\begin{align*}
\frac{\partial \lossaux_i}{\partial \bz}
&= \frac{\partial (\loss_i + \langle \bg_i, \bx_i \rangle)}{\partial \bz}  
= \frac{\partial \loss_i}{\partial \bz} + \frac{\partial \langle \bg_i, \bx_i \rangle}{\partial \bx_i } \frac{\partial \bx_i}{\partial \bz} \\
&= \frac{\partial \loss_i}{\partial \bz} + \bg_i \frac{\partial \bx_i}{\partial \bz}   
= \frac{\partial \loss_i}{\partial \bz} + \frac{\partial (\sum_{j=i+1}^{K} \loss_j)}{\partial \bx_i} \frac{\partial \bx_i}{\partial \bz} \\  
&= \frac{\partial \loss_i}{\partial \bz} + \frac{\partial (\sum_{j=i+1}^{K} \loss_j)}{\partial \bz}   
= \frac{\partial (\sum_{j=i}^{K} \loss_j)}{\partial \bz}.
\end{align*}
The above lines simply follow the definition of $\lossaux_i$ and the chain rule.
This concludes our proof of Eq.~\eqref{eq:auxiliary_loss_prop} for Stage $i$, and thus our proof for the proposition.
\end{proof}

Let us move on to relax the assumption.
In the broader scenario with tied model parameters, e.g.~word embedding matrices \cite{Press2017}, across pipeline stages, 
gradient calculation via backpropagation is equivalent to the following two-step procedure: 
(1)~compute gradients \emph{as if} all parameters are untied, then 
(2)~sum up and synchronize gradients for tied parameters via all-reduce operations.
Hence our proposed auxiliary-loss approach, when applied to the first part of this two-step procedure, is still valid.

\begin{figure*}
\centering
\vspace{-0.5em}
\includegraphics[width=.97\textwidth]{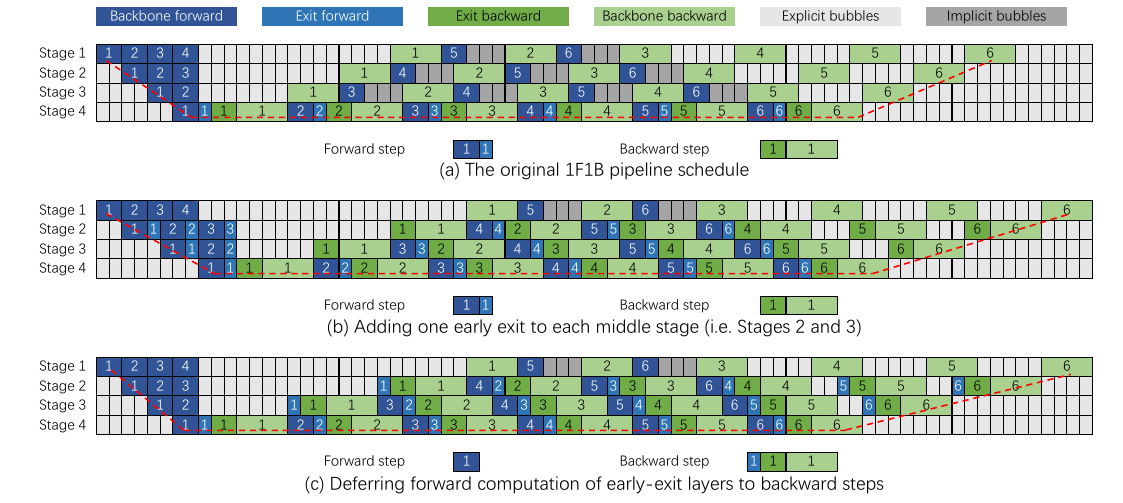}
\vspace{-1em}
\caption{One iteration of the 1F1B pipeline schedule, in a setting with $P=4$ pipeline stages and $M=6$ microbatches per batch.
At the top of this figure,
``Backbone forward/backward'' stands for computation of Transformer layers on the backbone, 
while ``Exit forward/backward'' stands for computation of early-exit or final-exit layers.
The number in each block denotes the index of the corresponding microbatch.
Critical paths are marked by dashed red lines.
From Figure~(a) to~(b), additional ``Exit forward/backward'' blocks are added, due to the introduction of early exits to middle stages.
From Figure~(b) to~(c), the order of computation is slightly adjusted for the purpose of reducing memory usage.
For clarity, we ignore computation related to the input embedding layer, and P2P communication latency between pipeline stages.
}
\vspace{-1em}
\label{fig:pipeline_schedules}
\end{figure*}

\subsubsection{Integration with the 1F1B schedule}

We have shown how to modify the forward and backward steps for \emph{each} microbatch, in order to execute backpropagation of the early-exit training loss through pipeline stages.
In principle, this lightweight modification can be integrated with general pipeline schedules of forward and backward steps for \emph{multiple} microbatches.
The classical 1F1B (one-forward-one-backward) schedule, also called PipeDream-Flush \cite{Narayanan2019,Fan2021,Narayanan2021}, achieves a good balance of algorithmic simplicity, training time, memory usage and communication latency, in comparison to other schedules such as GPipe \cite{Huang2019} (larger activation memory) or interleaved 1F1B \cite{Narayanan2021} (higher memory and communication requirements).
Hence for concreteness, we focus on the 1F1B schedule throughout this work.
Within one iteration of this schedule, each stage goes through a \emph{warm-up} phase (forward steps of the beginning microbatches), a \emph{steady} 1F1B phase, and a \emph{cool-down} phase (backward steps of the final microbatches).
A visualization can be found in Figure~\ref{fig:pipeline_schedules}(a).

\subsection{Training efficiency}
\label{subsec:training_efficiency_brief}

At first glance, one might expect that adding early exits to a standard LLM will incur a training overhead, in terms of time and memory, that is (at least) proportional to the number of additional model parameters.
Fortunately, this is not the case for training with pipeline parallelism.
We discover that computation related to the additional early exits can leverage idle resources in the original 1F1B schedule, 
thus causing \emph{negligible overhead to training time and zero overhead to peak GPU memory} under mild conditions.

To explain this, let us first identify three types of idle resources in the original 1F1B schedule.
\begin{itemize}
    \item \textbf{Explicit bubbles}, i.e.~the light gray areas in Figure~\ref{fig:pipeline_schedules}(a), during which GPUs are idle.
    \item \textbf{Implicit bubbles}, i.e.~the dark gray areas in Figure~\ref{fig:pipeline_schedules}(a), caused by load imbalance across pipeline stages.
    In particular, the last pipeline stage requires additional computation for the output layer.
    \item \textbf{Idle memory}, caused by unbalanced memory usage: earlier stages have to save the intermediate activations for more microbatches, and the first/last stage need to store an extra input/output layer, plus the corresponding gradients and optimizer states. As a result, the first stage is typically the bottleneck of peak memory usage.

\end{itemize}

Now, suppose that we choose $k$ middle stages and add one minimalistic early-exit layer to each of them.
As shown in Figure~\ref{fig:pipeline_schedules}(b),
by utilizing implicit bubbles,
computation of $k$ early exits will increase the training time per iteration by only $k \times (f_{\mathsf{EE}} + b_{\mathsf{EE}})$, 
where $f_{\mathsf{EE}}$ (resp.~$b_{\mathsf{EE}}$) represent the time for one forward (resp.~backward) pass of one microbatch for one early-exit layer.
Moreover, memory usage by model parameters for each middle stage will not exceed that of the first or last stage.
Deferring forward computation of early exits to backward steps, as shown in Figure~\ref{fig:pipeline_schedules}(c),
further decreases the memory overhead by early-exit logits in Stage $i \in [P]$ from $s \times b\times V \times (P - i + 1)$ to $s \times b\times V$, where
$s$ is the sequence length,
$b$ is the microbatch size, 
$V$ is the vocabulary size,
and $P - i + 1$ is the number of in-flight microbatches \cite{Korthikanti2022} for Stage $i$.
Putting these together, the peak memory usage across stages remains unchanged,
as long as $s \times b\times V$ is less than the activation memory of \emph{all} backbone Transformer layers within one stage for one microbatch.

Due to limited space, we defer detailed and formal analysis of training efficiency in broader settings to Appendix~\ref{sec:training_efficiency}.

\subsection{Advanced features}
\label{subsec:additional_features}

\sys incorporates some advanced features that can potentially improve the training process. 
One of them is the option of changing early-exit loss weights during training, just like learning rate or other hyperparameters.
Another one is a novel method of improving resource utilization by filling explicit bubbles 
with \emph{partial} forward and backward computation for \emph{additional microbatches}, which is visualized in Figure~\ref{fig:pipeline_fill_bubbles}.
One can prove formally that such additional computation provides useful gradient information for optimizing the early-exit training objective, without increasing the training time per iteration.
More details of both features can be found in Appendix~\ref{sec:appendix_advanced_features}.


\begin{figure}[tbp]
\centering
\includegraphics[width=\columnwidth]{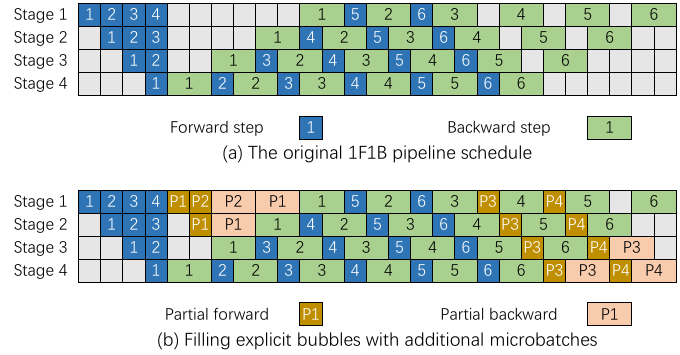}
\vspace{-1.5em}
\caption{
The proposed method of filling bubbles with additional microbatches.
In this example, P1 and P2 go through the forward and backward passes for the first few stages, 
while P3 and P4 go through the full forward pass, followed by the backward pass for the last few stages.
}
\label{fig:pipeline_fill_bubbles}
\end{figure}

\section{Inference}
\label{sec:inference}

This section first explains the major challenge faced by early-exit LLM inference for autoregressive generation, 
and a few recent attempts to resolve it.
Then, we introduce a novel solution based on a new type of pipeline parallelism.
Throughout this section, we focus on the latency-oriented setting, with a batch size of 1 during sequence generation.

\paragraph{Main challenge: KV caching.}

Accelerating inference with early exiting is, in principle, orthogonal to and compatible with many other common techniques of acceleration, such as kernel fusion, FlashAttention \cite{dao2022flashattention}, quantization \cite{Yao2022ZeroQuantEA,Xiao2023SmoothQuantAA}, among others \cite{Kim2023FullSO}.
For autoregressive generation, one exception is KV caching, i.e.~saving keys and values of all attention layers for previously generated tokens, which reduces redundant computation at the cost of higher memory usage.
This is contradictory to vanilla early-exit inference: if the current token is generated via early exiting, then its KV caches in later layers are missing, which hinders the generation of future tokens that go beyond the exiting layer of the current token.
This challenge has been well recognized in the literature, and several approaches have been recently proposed to resolve it,
including state propagation \cite{Elbayad2020,Li2021AcceleratingBI,Schuster2022}, SkipDecode \cite{DelCorro2023}, and synchronized parallel decoding \cite{Bae2023,tang2023deed}.
See Appendix~\ref{sec:related_works} for more discussion on these methods.

A variant of the last method, which we call \emph{KV recomputation}, is implemented in \sys. 
In this approach, we maintain a list of the most recent tokens that have missing KV caches in deep layers due to early exiting. 
During each forward pass, we include these early-exit tokens in the current forward pass, which allows for direct recomputation of the KV caches for these tokens and thus avoids the issue of missing KV caches.
It relies on the batching effect of GPU computation for early-exit acceleration,  
yet might not achieve any acceleration on other hardware platforms, due to its high theoretical complexity.

\begin{figure}[tbp]
\centering
\includegraphics[width=\columnwidth]{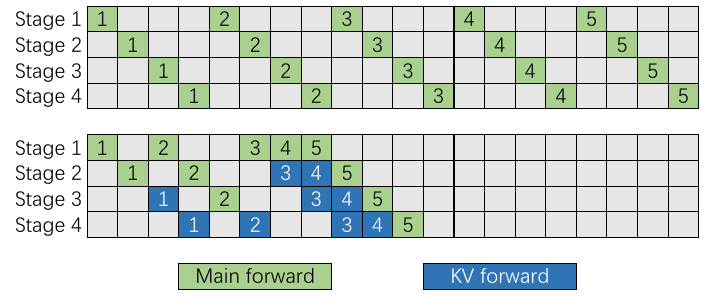}
\vspace{-1.0em}
\caption{
Standard full-model inference (top) and
our pipeline-based early-exit inference (bottom).
Numbers in the blocks denote the tokens within one generated sequence.
For simplicity of visualization,
we assume here that (1)~each early exit is located at the end of some pipeline stage, and
(2)~the latency for generating each token is the same (while in practice, generating the first token via the prefilling phase usually takes longer than generating another token during the decoding phase).
}
\label{fig:inference_pipeline}
\end{figure}

\paragraph{New solution: pipeline parallelism.}

We propose a novel type of pipeline parallelism to tackle the aforementioned challenge.
The key idea is that, in the process of inference with multiple pipeline stages, the following two processes run \emph{in parallel} whenever the model decides to do early exiting for the current token at a certain exit:
\begin{itemize}
    \item The generated token is sent back to the first stage, and the forward pass for generating the next token is started immediately;
    \item The full-model forward pass of the current token is continued from the exiting layer, which fulfills its KV caches in all later layers.
\end{itemize}
See Figure~\ref{fig:inference_pipeline} for a visualization of this approach.
Although each token essentially goes through a forward pass of the full model, the computation after the exiting layer is \emph{parallelized} with the computation of later tokens, which is how acceleration is achieved in this approach.
It can be checked that the inference latency for generating one token at a certain exit matches exactly the time needed for the forward computation before returning an output at that exit,
unless the selected exit is located in the middle of the first pipeline stage, in which case generation of the next token has to wait until the forward pass of the first stage for the current token is completed.
Note that this is true not just in practice but also for the \emph{theoretical} time complexity, without relying on the batching effect of GPU computation like KV recomputation does.
An implementation of this pipeline-based inference method is provided in \sys.
One potential limitation of the proposed method is that it requires multiple devices to facilitate pipeline parallelism, although parallelism within a single GPU or other device might be possible with more advanced implementation.

\section{Experiments}
\label{sec:experiments}

This section provides an empirical evaluation of the training and inference efficiency achieved by \sys.

\subsection{Training}
\label{subsec:exp_training}

In the following experiments, we empirically investigate the convergence of training early-exit models with \sys,
as well as the training efficiency of \sys for early-exit LLMs up to an unprecedented scale of 30B.
This scale is only limited by the hardware resources available to us, namely an 8-node cluster with 8 Nvidia A100-80GB GPUs in each node and hence 64 GPUs in total.
We use a random subset of the pre-training data provided by Data-Juicer \cite{DJdata,Chen2023}.
Experiments for models of the same size use the same subset and order of data.
In all cases, we use the Adam optimizer \cite{Kingma2014AdamAM} with $\beta_1 = 0.9, \beta_2 = 0.95, \epsilon = 10^{-8}$,
and a cosine schedule for the learning rate, with a maximum value of $3 \times 10^{-4}$.

\paragraph{Convergence of training losses.}

We first consider a 1.3B GPT Transformer with 24 layers, add one minimalistic early-exit layer without layer normalization to the 1/4 depth and the other to the 1/2 depth, set their early-exit loss weights to 1/4 and 1/2 respectively (while the final-exit loss has a weight of 1), and tie all input and output embedding matrices.
A standard LLM of the same architecture is trained using the same hyperparameters and pre-training data.
We further train a 7B early-exit model with 32 layers using similar configurations, except that early-exit loss weights are set to 0.1 and 0.2, and all embedding matrices are untied.
A standard 7B model is trained similarly.
The batch size and sequence length are both set to 2048.

\begin{figure}
\centering
\begin{subfigure}{.93\columnwidth}
    \centering
    \includegraphics[width=\textwidth]{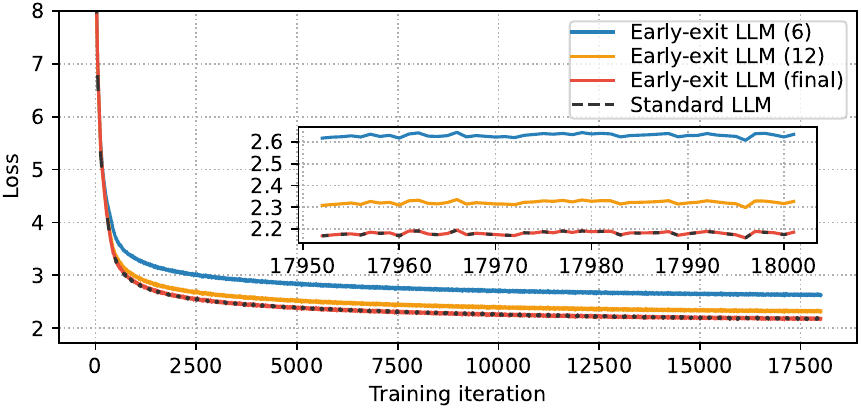}
    \caption{1.3B, 24 layers}
\end{subfigure}
\begin{subfigure}{.93\columnwidth}
    \centering
    \includegraphics[width=\textwidth]{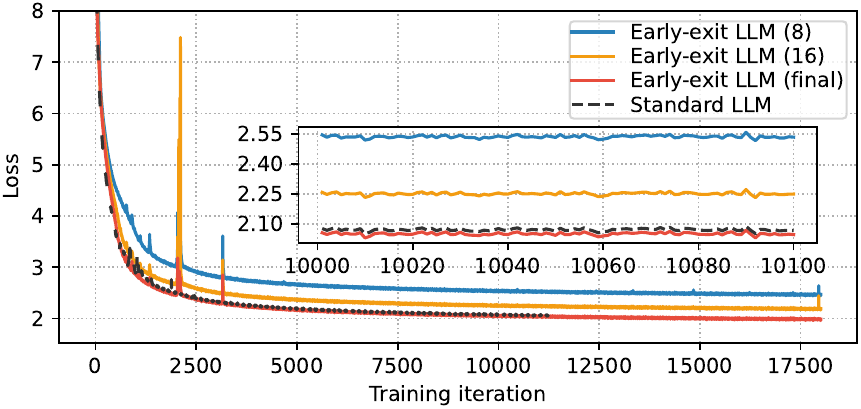}
    \caption{7B, 32 layers}
\end{subfigure}
\caption{Convergence of early-exit/final-exit training losses.
Each curve is annotated with the index of the Transformer layer that the corresponding exit is connected to.
}
\label{fig:training_loss}
\end{figure}

Figure~\ref{fig:training_loss} shows the convergence of early and final-exit training losses, i.e.~negative log-likelihood of next-token prediction, for both standard and early-exit LLMs.
All loss curves decay at a similar pace, and unsurprisingly, the early-exit losses are slightly higher than the final-exit loss for each model.
Interestingly, the final-exit loss curve of each early-exit model is close to (or even slightly below) that of the standard model,
suggesting that optimizing for early-exit losses might not hurt the full-model output in our setting.
We also observe more spikes in the loss curves for the 7B models than for the 1.3B models, 
possibly because
(1)~we choose to untie the embedding matrices and use smaller early-exit loss weights, both of which incur weaker regularization for the training process; and
(2)~layer normalization, which is known to stabilize the training of LLMs, is not included in the minimalistic early-exit layers of our 7B model.
Empirical results for the convergence of training losses under more general early-exit configurations can be found in Appendix~\ref{subsec:exp_training_losses_more_config}.

\begin{figure}
\centering

\begin{subfigure}{\columnwidth}
\centering
\includegraphics[width=.46\textwidth]{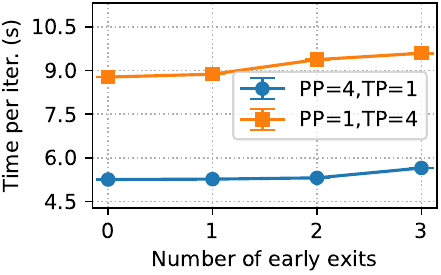}%
\includegraphics[width=.46\textwidth]{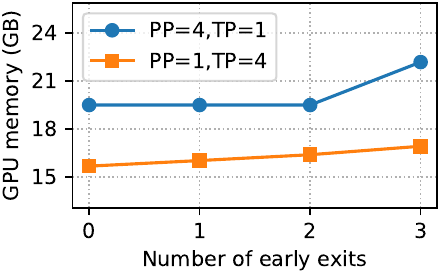}
\caption{1.3B}
\end{subfigure} \\

\begin{subfigure}{\columnwidth}
\centering
\includegraphics[width=.46\textwidth]{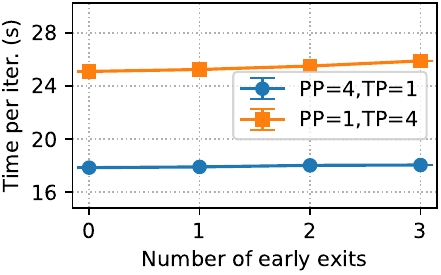}%
\includegraphics[width=.46\textwidth]{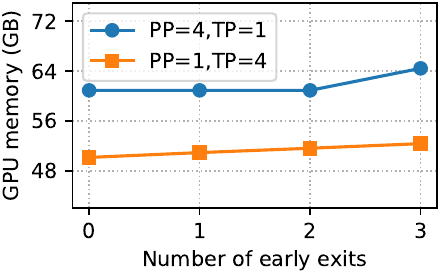}
\caption{7B}
\end{subfigure} \\

\begin{subfigure}{\columnwidth}
\centering
\includegraphics[width=.46\textwidth]{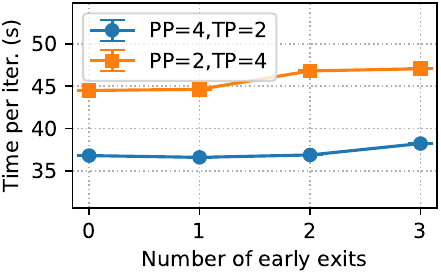}%
\includegraphics[width=.46\textwidth]{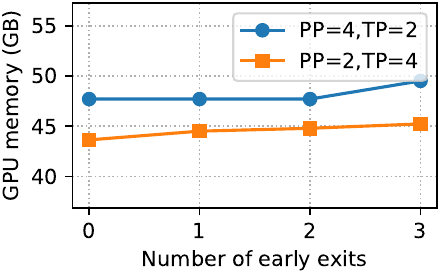}
\caption{13B}
\end{subfigure} \\

\begin{subfigure}{\columnwidth}
\centering
\includegraphics[width=.46\textwidth]{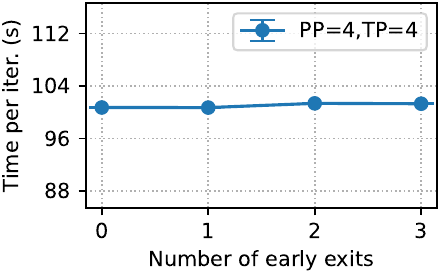}%
\includegraphics[width=.46\textwidth]{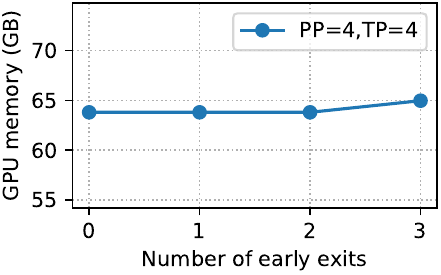}
\caption{30B}
\end{subfigure}
\vspace{-1.5em}
\caption{Training time per iteration and peak GPU memory vs.~the number of added early exits.
{Note that wall-clock time can be impacted by other workloads on the same GPU cluster when our experiments were conducted,
which inevitably causes perturbations to our numerical results.}
}
\label{fig:exp_training_efficiency}
\end{figure}

\paragraph{Training efficiency.}

Starting with a standard GPT Transformer of size ranging from 1.3B to 30B,
we increase the number of early exits from 0 to 3. 
Minimalistic early exits are added one by one to specific locations in the following order:
(1) to the 1/4 depth;
(2) to the 1/2 depth;
(3) to the hidden states right before the first Transformer layer, which is always located in the first pipeline stage.
Based on the performance optimizations proposed in Appendix~\ref{sec:training_efficiency}, when an early exit is inserted into the middle of two layers located on two consecutive pipeline stages, we always add it to the beginning of the latter stage.
We set the global batch size to 2048, microbatch size to 2 (for 1.3B and 7B models) or 1 (for 13B and 30B models), and sequence length to 2048.
The data parallelism degree is fixed at 4,
while tensor (TP) and pipeline (PP) parallelism degrees take various values.

The empirical results illustrated in Figure~\ref{fig:exp_training_efficiency}
matches our analytical study in Section~\ref{subsec:training_efficiency_brief}.
In particular, training time increases with the number of added early exits, 
but at a slower rate when pipeline parallelism is enabled, thanks to the proposed utilization of implicit bubbles.
Without pipeline parallelism, peak GPU memory increases with the number of early exits;
on the other hand, with the pipeline parallelism degree set to 4, peak memory remains unchanged as early exits are added to middle pipeline stages, and only increases when the last early exit is added to the first stage. 

\subsection{Inference}
\label{subsec:exp_inference}

In the experiments below, we verify the downstream performance of the pipeline-based approach proposed in Section~\ref{sec:inference}, with the number of pipeline stages set to 4. 
A server with 4 Nvidia A100-40GB GPUs is used for inference.
We use the 1.3B and 7B early-exit LLMs from the previous experiment on convergence of training losses in Section~\ref{subsec:exp_training},
which have been pre-trained from scratch using 300B and 150B tokens respectively,
without further fine-tuning or alignment.
Since designing the decoding method or exit condition is not our focus,
we simply adopt greedy decoding and a basic confidence-based exit condition: 
at each early exit, if the maximum probability of next-token prediction is above a pre-defined threshold, then the model chooses the most likely token and moves on to generate the next one.
The trade-off between output quality and speed is controlled by this threshold.
When the threshold is set to 1, early-exit layers are disabled, so that the time complexity of a full-model forward pass matches that of a standard LLM\footnote{We observed empirically that the actual wall-clock latency between these two cases might differ by up to 2\%, which is negligible compared to the early-exit speedup.}.
Inference latency in this case is used as the baseline for calculating relative speedup.
For benchmarking, we evaluate our models with HELM \cite{helm} in six tasks:
BoolQ \cite{clark2019boolq}, TruthfulQA \cite{truthfulqa}, NaturalQuestions open-book and closed-book \cite{naturalquestions}, XSUM \cite{xsum}, and CNN/DailyMail.
The first four tasks are question-answering tasks, while the last two are summarization tasks.
We use EM (ratio of exact match with the correct answers) as the metric for BoolQ and TruthfulQA, 
F1 for NaturalQuestions open-book/closed-book,
and ROUGE-L (a measure of similarity to the reference summaries) for XSUM and CNN/DailyMail.

\begin{figure}[tbp]
\centering
\includegraphics[width=.5\columnwidth]{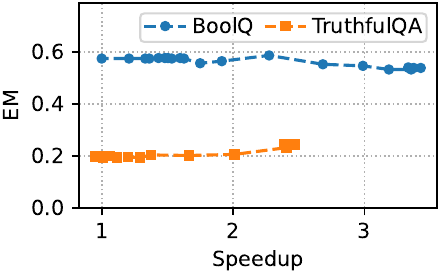}%
\includegraphics[width=.5\columnwidth]{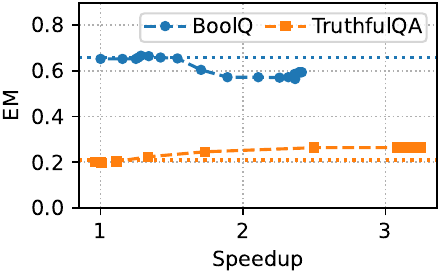}
\includegraphics[width=.5\columnwidth]{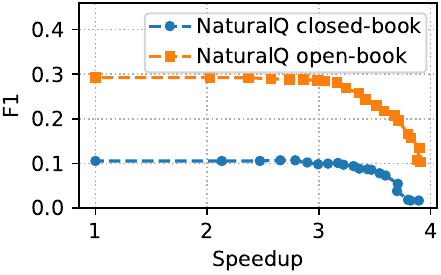}%
\includegraphics[width=.5\columnwidth]{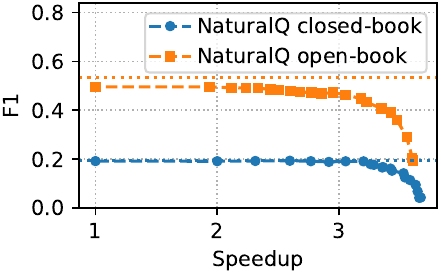}
\begin{subfigure}{.5\columnwidth}
    \centering
    \includegraphics[width=\textwidth]{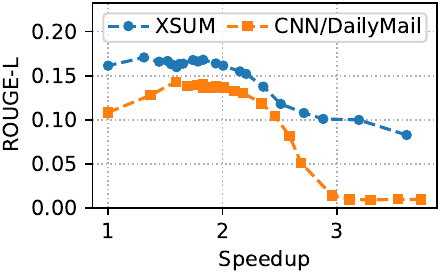}
    \caption{1.3B}
\end{subfigure}%
\begin{subfigure}{.5\columnwidth}
    \centering
    \includegraphics[width=\textwidth]{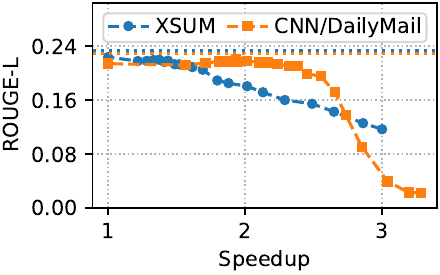}
    \caption{7B}
\end{subfigure}
\caption{Evaluation scores and relative speedup with our early-exit models. 
For each curve, speedup increases as the confidence threshold decreases from left to right.
The right column also includes horizontal dotted lines that represent the performance of a 7B standard LLM trained under the same configurations as those for our 7B early-exit model.
}
\label{fig:eval_em_rouge}
\end{figure}

Figure~\ref{fig:eval_em_rouge} demonstrates the scores and speedup for pipeline-based inference with varying confidence thresholds.
Encouragingly, in many cases, inference with early exiting achieves $2\times$ or higher speedup compared to full-model inference, with comparable or even better evaluation scores.
Table~\ref{tab:example_inference_speed} provides example texts generated by our 7B early-exit model, showing that speedup can be achieved with minor or no impact on the resulting sequence.
Table~\ref{tab:example_inference_prob} further demonstrates the confidence at each exit for each token, confirming the existence of easy tokens that can be predicted correctly with high confidence at early exits. 
See also Appendix~\ref{subsec:exp_two_inference_methods} for an empirical comparison between the pipeline-based method and KV recomputation, which are shown to perform similarly in our setting.
It is worth noting that performance of early-exit inference might be improved in many ways,
e.g.~using early exits of different structures,
other decoding mechanisms and exit conditions,
or simply more extensive and sufficient pre-training.
We also conjecture that early exiting can, in principle, bring higher speedup for larger LLMs, as expressivity of early exits grows with the model capacity, which allows them to make confident and accurate predictions for a larger proportion of tokens.

\section{Conclusions}
\label{sec:conclusions}

We have introduced \sys, a system for large-scale training and inference of early-exit LLMs with 3D parallelism.
For training, we have presented how to execute backpropagation of the early-exit training objective across pipeline stages,
performance optimizations for minimizing the computational overhead compared to training standard LLMs,
and advanced features for fine-grained control and optimization of the training process.
For inference, we have introduced our design and implementation of two inference methods, one based on KV recomputation and the other based on a new type of pipeline parallelism, both of which are compatible with KV caching for autoregressive generation.
Along the way, we have discovered some interesting chemical reactions between early exiting and pipeline parallelism,
which is likely because it is both \emph{along the depth dimension} that an early-exit model executes forward computation selectively and gets partitioned into pipeline stages.
We hope that \sys will be a helpful tool for future study and applications of early-exit LLMs at scales as large as those of standard LLMs.
It is worth noting that many ideas presented in this work can be extended to more general neural network architectures, deep learning frameworks, hardware platforms, or other broader settings.
It would be exciting to see further developments in these aspects.

\section*{Acknowledgements}
We would like to thank the reviewers for their constructive and valuable feedback that helps improve this work.

\section*{Impact Statement}
This paper presents work whose goal is to advance techniques for training and inference of early-exit large language models.
Therefore, it bears potential societal consequences similar to those of prior literature on the technical aspects of LLM training and inference, e.g.~misuse of the resulting pre-trained LLMs.
Other than that, we see nothing that we feel must be specifically highlighted here.


\bibliography{refs_multi_exit}
\bibliographystyle{icml2024}

\newpage
\appendix
\onecolumn

\section*{Structure of the appendix}

This appendix is organized as follows.
Appendix~\ref{sec:training_efficiency} supplements a detailed and formal analysis of training efficiency,
while Appendix~\ref{sec:appendix_experiments} provides additional experiments and empirical results.
Appendices~\ref{sec:appendix_advanced_features} and ~\ref{sec:implementation}
include details about advanced features and implementation of \sys.
Extended preliminaries and related works can be found in Appendices~\ref{sec:preliminaries} and~\ref{sec:related_works}.

\section{Training efficiency}
\label{sec:training_efficiency}

This section is an extended version of Section~\ref{subsec:training_efficiency_brief}.
In the following, we analyze the training time and peak GPU memory of training an early-exit model with pipeline parallelism, and propose some performance optimizations.
We refer readers to the literature (e.g.~the Megatron-LM series \cite{Shoeybi2019,Narayanan2021,Korthikanti2022}) for thorough study of training efficiency with 3D parallelism for standard LLMs without early exits; 
we will use that as a baseline, and focus on the additional computational overhead caused by early exits.

Let us first identify the major sources of low resource utilization in the original 1F1B pipeline schedule for training a standard LLM, which lays the foundation for our analysis later.
\begin{itemize}
    \item \textbf{Explicit bubbles,}
    i.e.~light gray areas in Figure~\ref{fig:pipeline_schedules}(a), during which GPUs are idle.
    This is the most notable and well recognized source of low resource utilization in pipeline parallelism.
    \item \textbf{Implicit bubbles,} 
    i.e.~dark gray areas in Figure~\ref{fig:pipeline_schedules}(a).
    This is caused by load imbalance across pipeline stages, even though Transformer layers are evenly divided into stages in Megatron-LM\footnote{For simplicity, we do not consider other flexible ways of dividing an LLM into pipeline stages, as dividing Transformer layers evenly remains one of the most practical and robust options in practice.
    We also do not consider the case where the input/output embedding matrix occupies a separate pipeline stage, although our techniques for training an early-exit LLM can be generalized to this case in principle.}.
    In particular, the first stage has the additional computation for the input embedding layer, 
    and more importantly, the last pipeline stage has the additional computation for the output logits (via the output embedding layer) as well as the training loss.
    For LLMs, these additional computational costs are not negligible, primarily due to large vocabulary sizes.
    \item \textbf{Idle memory.}
    Memory usage is also unbalanced across pipeline stages:
    earlier stages in the 1F1B schedule have to save the intermediate activations for more microbatches,
    and the first/last stage has to save an extra input/output embedding layer, plus the corresponding gradients and optimizer states.
    As a result, the first stage is typically the bottleneck of peak memory usage, while later stages have idle memory \cite{Korthikanti2022}.
\end{itemize}

\begin{remark}
In our analysis, we assume no activation recomputation \cite{chen2016,Korthikanti2022} for simplicity. 
For clarity, we also ignore the P2P communication latency between pipeline stages, which is generally not the major concern for the efficiency of pipeline parallelism.
\end{remark}

\subsection{Utilization of idle resources}
\label{subsubsec:util_idle_resources}

At first glance, one might expect that adding early exits to an LLM will incur a training overhead, in terms of time and memory, that is (at least) proportional to the number of additional model parameters.
Fortunately, this is not the case for training with pipeline parallelism, based on the above analysis of idle resources in the 1F1B pipeline schedule.
Indeed, adding one minimalistic early-exit layer (which has the same structure as the final output layer) to some middle (i.e.~not the first or last) stage will only make its model size and theoretical forward/backward time match exactly those of the last stage.
Therefore, the aforementioned implicit bubbles and some of the idle memory can be automatically utilized for computation related to the early-exit layers, leading to more balanced load across pipeline stages.

More specifically, the overhead to training time caused by additional early exits can be negligible.
If we choose $k$ middle stages and add one minimalistic early-exit layer to each of them, then the training time per iteration, i.e.~time for processing one data batch, will (in theory) increase only by $k \times (f_{\mathsf{EE}} + b_{\mathsf{EE}})$, where $f_{\mathsf{EE}}$ and $b_{\mathsf{EE}}$ represent the time needed for one forward and backward pass of one microbatch for one minimalistic early-exit layer, respectively.
To see this, first notice that the computation of early-exit layers in the steady 1F1B phase can perfectly fit into the implicit bubbles.
Therefore, the critical path remains the same as in the case without early exits, which consists of 
(1) the forward steps on all stages for the first microbatch, 
(2) the 1F1B steady phase on the last stage, and
(3) the backward steps on all stages for the last microbatch.
The early-exit layers, located separately on $k$ middle stages, will cause a $k \times f_{\mathsf{EE}}$ overhead to the first part of the critical path, and a $k \times b_{\mathsf{EE}}$ overhead to the third part, leading to the aforementioned claim on the overhead to training time.
See Figures~\ref{fig:pipeline_schedules}(a) and (b) for a visualization.

\subsection{Further performance optimizations}
\label{subsubsec:further_perf_opt}

\paragraph{Reducing activation memory overhead.}

So far, we have analyzed training time and memory usage by model parameters. 
Another major component of memory usage, namely the memory for gradients and optimizer states, 
can be bounded by the memory for model parameters multiplied by a universal constant.
One remaining issue that we have not addressed is the \emph{activation memory} overhead due to early-exit computation.
Most notably, the early-exit logits for one microbatch have size $s \times b \times V$, where 
$s$ is the maximum sequence length,
$b$ is the microbatch size, and 
$V$ is the vocabulary size.
If the $i$-th stage has one early exit, then a vanilla implementation of training with early exits using the 1F1B pipeline schedule (as shown in Figure~\ref{fig:pipeline_schedules}(b)) will cause a significant memory overhead of size $s \times b\times V \times (P - i + 1)$, where $P - i + 1$ is the number of in-flight microbatches \cite{Korthikanti2022} for Stage $i$.

Our solution for resolving this issue is simple: deferring the forward computation of each early-exit layer for each microbatch from the forward step to the corresponding backward step.
Note that this is feasible, because forward computation of the next stage only requires as input the hidden states returned by the current stage, while the results of early-exit forward computation are optional.
By adjusting the order of computation in this way, it is guaranteed that the early-exit logits for each microbatch are generated, used, and discarded immediately, within the same backward step; 
consequently, the activation memory overhead due to early-exit logits is reduced from $s \times b\times V \times (P - i + 1)$ to $s \times b\times V$.
As long as this amount of memory usage is less than the activation memory of \emph{all} Transformer layers within one stage for one microbatch (and no early exit is added to the first stage), the peak memory across all pipeline stages will stay unchanged, since the first stage remains the bottleneck of memory usage.

\begin{remark}
One can check that, with the above adjustment, our analysis in Section~\ref{subsubsec:util_idle_resources} about the training time overhead caused by early exits remains valid after minor modifications.
More specifically, the time overhead of one training iteration is still $k \times (f_{\mathsf{EE}} + b_{\mathsf{EE}})$; 
the only difference is that this whole overhead comes from the backward steps of the last microbatch, i.e.~the third part of the critical path in Figure~\ref{fig:pipeline_schedules}(c).
One can further reduce this overhead to $k \times b_{\mathsf{EE}}$ by moving the forward pass of the early-exit layer on each stage for each cool-down microbatch to the explicit bubble in front of the corresponding backward step (i.e.~before communication with the next stage).
With that said, our implementation does not include this modification, as it brings limited gains at the cost of complicating our codebase.
\end{remark}

\paragraph{Some rules of thumb.}

Below are some tips for maximizing training efficiency.
\begin{itemize}
    \item If possible, add early exits to the middle stages rather than to the first or last one.
    For example, adding one early exit to the end of the first stage leads to the same model architecture as adding to the beginning of the second stage, but the latter has higher training efficiency due to more balanced load across stages.
    \item Avoid adding too many early exits to the LLM.
    Despite higher flexibility during inference, the gain of adding many early exits (e.g.~one per layer) might be marginal, and comes at the cost of excessive overhead for training and inference, which is especially the case for LLMs due to large vocabulary sizes.
    Similar observations and advice have been made recently by the authors of \cite{Bae2023} as well.
    \item If there are multiple exits within the same pipeline stage, one might use the same output embedding matrix for all exits; 
    similarly, if early exits are added to the first/last stage, one might reuse the original input/output embedding matrix for early exits. 
    These choices reduce the memory usage by model parameters, at the cost of lower expressivity of early exits.
\end{itemize}

\begin{remark}
Recall from Section~\ref{sec:overview} that,
with \sys, users can choose more expressive and powerful early-exit layers beyond the minimalistic structure.
Similarly, more than one early exit can be added to each pipeline stage, which provides more flexible choices of exits during inference.
These benefits, of course, come at the cost of a higher overhead for training, and potentially for inference as well\footnote{
There is no clear answer to whether additional modules at early-exit layers will improve or hurt the overall inference speed. There is certainly a higher overhead for the computation of each early-exit layer; 
on the other hand, higher flexibility and adaptivity of the early exits can potentially enable them to produce better outputs and get selected more often during inference, leading to overall faster generation of a complete sequence.
For similar reasons, there is no clear positive or negative correlation between the number of early exits and the overall speed of generating a sequence.};
with \sys, users can conveniently choose the most suitable configurations for their own use cases.
A formal analysis of training efficiency in these general cases will soon be presented in Appendix~\ref{subsec:analysis_efficiency}.
\end{remark}

\begin{figure*}
\centering
\includegraphics[width=.33\textwidth]{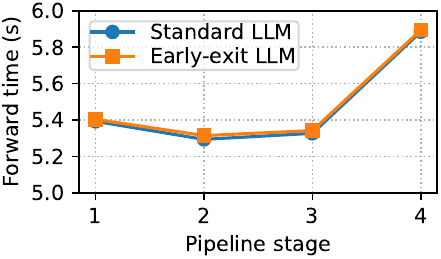}%
\includegraphics[width=.33\textwidth]{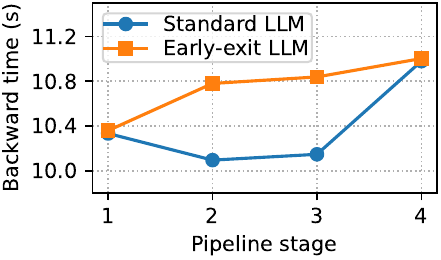}%
\includegraphics[width=.33\textwidth]{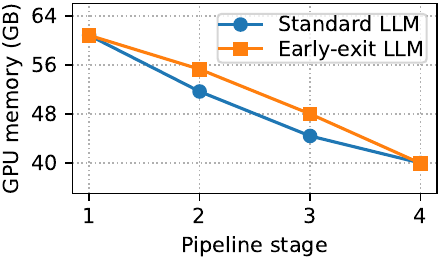}
\caption{
The forward time, backward time, and peak GPU memory of each pipeline stage for a standard 7B GPT Transformer, as well as its early-exit version that has one minimalistic early-exit layer (without layer normalization) added to each middle stage.   
Degrees of pipeline, tensor and data parallelism are 4, 1, and 1, respectively; 
the microbatch size is 2, the global batch size is 128, and the sequence length is 2048.
Note that the forward computation of early-exit layers has been deferred to the backward steps, hence not included in ``forward time'' of the first plot, but in ``backward time'' of the second plot.}
\label{fig:pipeline_load_imbalance}
\end{figure*}

\begin{table*}
\centering
\caption{Training efficiency and impacts of performance optimizations, with the same setting as in Figure~\ref{fig:pipeline_load_imbalance}. 
For the ``Early-exit'' row, the early exit at 1/4 depth is added to the end of Stage 1, and the exit at 1/2 depth is added to the end of Stage 2.
The last three rows are annotated with the performance optimization(s) adopted, where
Optimization 1 stands for deferring forward computation of early-exit layers to backward steps,
and Optimization 2 stands for moving every early exit from the end of some pipeline stage to the beginning of the next stage, in order to achieve more balanced load across stages. 
Stage 1 is the bottleneck of peak memory in most cases, except for the numbers marked by *, for which Stage 2 is the bottleneck.}
\label{tab:example_training_efficiency}
\small
\begin{tabular}{lSSSS}
    \toprule
    \multirow{2}{*}{Setup} &
    \multicolumn{2}{c}{1.3B} &
    \multicolumn{2}{c}{7B} \\
    & {Time per iteration (s)} & {Peak memory (GB)} & {Time per iteration (s)} & {Peak memory (GB)} \\
    \midrule
    Standard            & 5.23 & 19.85          & 17.75  & 62.27 \\
    Early-exit          & 5.31 & 24.05          & 17.93  & 67.42 \\
    Early-exit (1)      & 5.29 & 22.56          & 17.91  & 65.79 \\
    Early-exit (2)      & 5.28 & 20.23 {\ *}    & 17.81  & 62.27 \\
    Early-exit (1\&2) & 5.24 & 19.85          & 17.79  & 62.27 \\
    \bottomrule
\end{tabular}
\end{table*}

\paragraph{Numerical examples.}

We complement previous analytical study with a few numerical examples.
Figure~\ref{fig:pipeline_load_imbalance} illustrates load imbalance in the original 1F1B pipeline schedule for a standard 7B GPT Transformer, as well as the impacts of adding one minimalistic early-exit layer to each middle stage (with all performance optimizations applied).
Table~\ref{tab:example_training_efficiency} takes a close look at the impacts of each performance optimization;
unsurprisingly, the best training efficiency is achieved with all the proposed optimizations applied.

\subsection{Theoretical analysis for general cases}
\label{subsec:analysis_efficiency}

In the following, we derive formulas for the training time per iteration and peak GPU memory usage across pipeline stages under general configurations of early exits,
in terms of the quantities listed in Table~\ref{tab:symbols_for_analysis}.
Some of these quantities, especially $\{f_{\osf}, b_{\osf}, m_{\osf}, \mact_{\osf}\}$, can be further calculated analytically with lower-level quantities such as sequence length, microbatch size, hidden size, vocabulary size and others \cite{Rajbhandari2019ZeROMO,Narayanan2021,Korthikanti2022},
or estimated empirically by profiling in practice.

It is assumed that all pipeline stages have the same number of Transformer layers on the backbone,
and all early-exit layers follow the same structure.
We also assume that input and output embedding matrices are untied, although it is not hard to extend our analysis to more general cases.
For simplicity, we ignore the point-to-point communication latency between pipeline stages in our analysis,
since it typically takes a minor proportion of the overall training time, and can often be overlapped with computation.

\begin{table}
\centering
\caption{A list of notation for theoretical analysis.
The following abbreviations are used:
\textsf{IN} --- \underline{in}put processing layer;
\textsf{EE} --- \underline{e}arly-\underline{e}xit layer;
\textsf{FE} --- \underline{f}inal-\underline{e}xit layer;
\textsf{BB} --- Transformer \underline{b}ack\underline{b}one on one pipeline stage.
}
\label{tab:symbols_for_analysis}
\small
\begin{tabular}{ll}
    \toprule
    Notation  &  Definition \\
    \midrule
    $P$  &  Number of pipeline stages \\
    $N_i$  &  Number of early exits in Stage $i \in [P]$ \\
    $M$               &  Number of microbatches for one training iteration \\
    $f_{\osf}, b_{\osf}$ & Forward and backward time of one $\osf \in \{\IN, \EE, \FE, \BB\}$ for one microbatch \\
    $m_{\osf}$     &  Memory usage for storing the model parameters of one $\osf \in \{\IN, \EE, \FE, \BB\}$ \\
    $\mact_{\osf}$ &  Activation memory of one $\osf \in \{\IN, \EE, \FE, \BB\}$ for one microbatch \\
    $\ind(\cdot)$ & The indicator function, $\ind(\mathcal{E}) = 1$ if the event $\mathcal{E}$ holds true, and $0$ otherwise \\
    \bottomrule
\end{tabular}
\end{table}

\subsubsection{Training time per iteration}
\label{subsubsec:analysis_training_time}

\newcommand{\itertime}{\text{time per iteration}}

Let us analyze the training time per iteration step by step.

\paragraph{Step 1: simplified analysis without early exits.}

We start by recalling the simplified analysis in prior works for the standard 1F1B pipeline schedule \cite{Narayanan2021},
which assumes that all stages have the same forward time $f$ and backward time $b$ for one microbatch.
Recall from Figure~\ref{fig:pipeline_schedules} that the critical path of one training iteration consists of three parts:
\begin{itemize}
    \item Part 1: forward steps of the first microbatch on all stages except the last one, which takes time $(P-1) \times f$; 
    \item Part 2: the steady 1F1B phase with all $M$ microbatches on the last stage, which takes time $M \times (f+b)$;
    \item Part 3: backward steps of the last microbatch on all stages except the last one, which takes time $(P-1) \times b$.
\end{itemize}
Taking the sum of these three parts, we have
\begin{equation*}
    \itertime = (P-1) \times f + M \times (f + b) + (P-1) \times b 
    = \underbrace{(P-1) \times (f+b)}_{\text{Parts 1 and 3}} + \underbrace{M \times (f + b)}_{\text{Part 2}}. 
\end{equation*}

\paragraph{Step 2: fine-grained analysis without early exits.}

We provide a more fine-grained analysis, by considering the computation related to the input processing layer ($\IN$) and final-exit layer ($\FE$) separately from the Transformer backbone ($\BB$).
It is reasonable to assume that $f_{\IN} < f_{\FE}$ and $b_{\IN} < b_{\FE}$, 
which implies that the last stage is the bottleneck of forward and backward time.
In this setting, 
Parts 1 and 3 of the critical path takes time $f_{\IN} + (P-1) \times f_{\BB}$ and $b_{\IN} + (P-1) \times b_{\BB}$, respectively.
Similarly, Part 2 now takes time $M \times (f_{\BB} + b_{\BB} + f_{\FE} + b_{\FE})$.
Taking the sum, we arrive at
\begin{equation*}
    \itertime = \underbrace{f_{\IN} + b_{\IN} + (P-1) \times (f_{\BB} + b_{\BB})}_{\text{Parts 1 and 3}}
    + \underbrace{M \times (f_{\BB} + b_{\BB} + f_{\FE} + b_{\FE})}_{\text{Part 2}}.
\end{equation*}

\paragraph{Step 3: fine-grained analysis with early exits.}

Finally, we are ready for our analysis in the setting with early exits ($\EE$).
First, it can be checked that early-exit layers incur a total overhead of 
$${\sum_{i \in [P-1]} N_i \times (f_{\EE} + b_{\EE})}$$ 
to Parts 1 and 3.
In addition, the sum of forward and backward time of one microbatch for Stage $i$ becomes
\begin{equation*}
f_{\BB} + b_{\BB} + \ind(i=1) \times (f_{\IN} + b_{\IN}) + \ind(i=P) \times (f_{\FE} + b_{\FE}) + {N_i \times (f_{\EE} + b_{\EE})}.
\end{equation*}
Note that Part 2, namely the steady 1F1B phase on the last stage, is now \emph{bottlenecked} by the maximum forward and backward time across all stages.
Putting things together, we have the following upper bound:
\begin{align*}
&\itertime \\
&\quad \le f_{\IN} + b_{\IN} + (P-1) \times (f_{\BB} + b_{\BB}) + {\color{red}\sum_{i \in [P-1]} N_i \times (f_{\EE} + b_{\EE})}    \\
&\quad\quad + M \times {\max_{i \in [P]}} \Big\{ f_{\BB} + b_{\BB} + \ind(i=1) \times (f_{\IN} + b_{\IN}) + \ind(i=P) \times (f_{\FE} + b_{\FE}) + {\color{red}N_i \times (f_{\EE} + b_{\EE})} \Big\}.
\end{align*}
The overhead caused by early-exit layers has been highlighted.
The first term corresponds to Parts 1 and 3 of the critical path, and the second term corresponds to Part 2.
It is worth mentioning that the second term of the overhead might not take effect due to the maximum operator, 
in which case the early-exit overhead to one training iteration is independent of the number of microbatches.


\subsubsection{Peak GPU memory}
\label{subsubsec:analysis_training_memory}

\newcommand{\memory}{\text{memory}}

Consider the GPU memory for one specific pipeline stage, say Stage $i$.
Recall that GPU memory usage during training is mainly composed of memory for model parameters, gradients, optimizer states, and intermediate activations \cite{Rajbhandari2019ZeROMO,Korthikanti2022}.
Among them, the memory for gradients and optimizer states (for many common optimizers, such as SGD with momentum or Adam) is proportional to the number of model parameters.
Therefore, we assume that 
$$
\memory(\text{model parameters, gradients, optimizer states}) 
= \alpha \times \memory(\text{model parameters})
$$
for some universal constant $\alpha > 1$,
whose concrete value depends on the choice of optimizer and numerical precisions.
Under this assumption, we have
\begin{equation}
\text{total memory} \approx \alpha \times \memory(\text{model parameters}) + \memory(\text{activations}).
\label{eq:total_memory}
\end{equation}

Now it boils down to deriving the memory for model parameters and activations.
Model parameters in Stage $i \in [P]$ include the Transformer backbone ($\BB$),
the input processing layer ($\IN$) if $i = 1$, the final-exit layer ($\FE$) if $i = P$,
and the early exits ($\EE$).
In other words, we have
\begin{equation}
\memory(\text{model parameters}) = m_{\BB} + \ind(i=1) \times m_{\IN} + \ind(i=P) \times m_{\FE} + {\color{red}N_i \times m_{\EE}}.
\label{eq:memory_model_parameters}
\end{equation}
The memory for intermediate activations can be similarly divided into components 
corresponding to $\BB$, $\IN$, $\FE$ and $\EE$.
Note that according to Section~\ref{subsubsec:further_perf_opt}, 
the peak memory usage by activations within one early-exit layer is only $\mact_{\EE}$ 
regardless of the number of in-flight microbatches $P + 1 - i$.
Thus one has
\begin{equation}
\memory(\text{activations}) =  
(P + 1 - i) \times \mact_{\BB} + \ind(i=1) \times P \times \mact_{\IN} + \ind(i=P) \times \mact_{\FE} + {\color{red}N_i \times \mact_{\EE}}.
\label{eq:memory_activations}
\end{equation}
Combining Eq.~\eqref{eq:memory_model_parameters} and~\eqref{eq:memory_activations} with Eq.~\eqref{eq:total_memory}, 
we arrive at an estimate of memory usage for Stage $i$:
\begin{align*}
\text{total memory} \approx \alpha \times &\Big( m_{\BB} + \ind(i=1) \times m_{\IN} + \ind(i=P) \times m_{\FE} + {\color{red}N_i \times m_{\EE}} \Big) \\
+ &\Big( (P + 1 - i) \times \mact_{\BB} + \ind(i=1) \times P \times \mact_{\IN} + \ind(i=P) \times \mact_{\FE} + {\color{red}N_i \times \mact_{\EE}} \Big).
\end{align*}
The memory overhead caused by early-exit layers has been highlighted in red.
Finally, taking the maximum over $i \in [P]$ gives the peak GPU memory across all pipeline stages.

\section{Additional experiments}
\label{sec:appendix_experiments}

\subsection{A comparison between two inference methods}
\label{subsec:exp_two_inference_methods}

Since both pipeline-based method and KV recomputation generate the same output for the same prompt,
here we only compare the inference latency of both methods.
Given that the pipeline-based method uses 4 GPUs for a pipeline parallelism (PP) degree of 4, we allow KV recomputation to use a tensor parallelism (TP) degree of 1 or 4 for a fair comparison.
The empirical results on XSUM and CNN/DailyMail tasks are shown in Figure~\ref{fig:compare_inference_methods}.
We first notice that the pipeline-based approach outperforms KV recomputation with $\text{TP}=1$ for confidence thresholds smaller than 1 (i.e.~when early exiting actually happens), 
despite the point-to-point communication overhead.
Moreover, KV recomputation with $\text{TP}=4$ is faster than the pipeline-based approach;
it is worth noting, however, that the small speedup from $\text{TP}=1$ to $\text{TP}=4$ for KV recomputation is only possible with high-end hardware like A100 GPUs connected by high-bandwidth communication.
In many use cases of LLM inference, such hardware is not available, and pipeline parallelism is the only option for partitioning a model that is too large to fit into the memory of one single device.
In sum, the proposed pipeline-based approach is a more practical option for LLM inference with model partitioning in a wide range of scenarios.

\begin{figure}[tbp]
\centering
\begin{subfigure}{.25\textwidth}
    \centering
    \includegraphics[width=\textwidth]{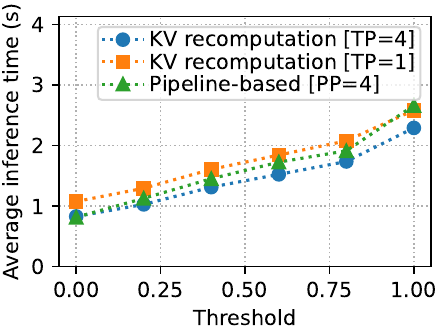}
    \caption{XSUM}
\end{subfigure}%
\begin{subfigure}{.25\textwidth}
    \centering
    \includegraphics[width=\textwidth]{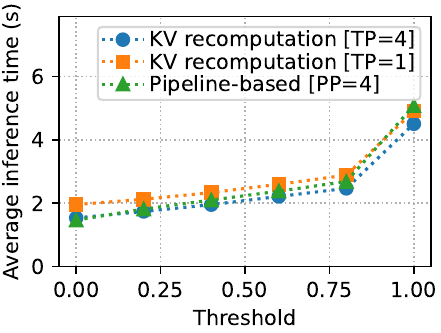}
    \caption{CNN/DailyMail}
\end{subfigure} 
\caption{Inference latencies of the pipeline-based method and KV recomputation for our 7B early-exit model.}
\label{fig:compare_inference_methods}
\end{figure}

\subsection{Examples of generated texts}
\label{subsec:examples_generated_texts}

Example texts generated by our early-exit models are presented in Tables~\ref{tab:example_inference_speed} and~\ref{tab:example_inference_prob}.

\begin{table}
\centering
\caption{Sequences generated by our 7B early-exit model for the same prompt, with varying confidence thresholds and inference latencies.
Differences from the result of full-model inference are highlighted in red.}
\label{tab:example_inference_speed}
\small
\begin{tabular}{p{.2\textwidth}p{.7\textwidth}}
\toprule
Prompt  & Artificial General Intelligence is    \\
\midrule
Full model \newline(Time: \textbf{1.42} s)   &  {a branch of computer science that studies the development of computer systems that 
can perform tasks that normally require human intelligence. \newline
The term "artificial intelligence" was coined by John McCarthy in 1956.}   \\
\midrule
$\text{Threshold} = 0.8$ \newline(Time: \textbf{1.17} s)  &  {a branch of computer science that studies the development of computer systems that can perform tasks that normally require human intelligence. \newline
The term "artificial intelligence" was coined by John McCarthy in 1956.}   \\
\midrule
$\text{Threshold} = 0.4$ \newline(Time: \textbf{0.89} s)  &  {a branch of computer science that studies the development of computer systems that can perform tasks that normally require human intelligence. \newline
The term {\color{red}artificial general intelligence} was coined by John McCarthy in {\color{red}1965.}}  \\
\midrule
$\text{Threshold} = 0.2$ \newline(Time: \textbf{0.58} s)  &  {a branch of computer science that deals with the development of computer programs that can {\color{red}be used to perform tasks such as problem solving, decision making, and 
learning. \newline
Artificial General Intelligence is a branch of computer science that}}  \\
\bottomrule
\end{tabular}
\end{table}

\begin{table}
\centering
\caption{The predicted token and confidence at each exit of our 32-layer 7B early-exit model for generating each token within one sequence.
The prompt is ``The capital of China is'',
and the sequence generated by full-model inference is
``Beijing.
\textbackslash n
The capital of the United States is Washington, D.C.''.
Tokens that are predicted with confidence higher than 0.8 are highlighted; interestingly, all exits make the same prediction for each of these tokens.}
\label{tab:example_inference_prob}
\small
\begin{tabular}{p{.12\textwidth}p{.12\textwidth}p{.12\textwidth}p{.12\textwidth}p{.12\textwidth}p{.12\textwidth}}
\toprule
Layer 8  &  &  Layer 16  &  &   Final layer  &   \\
\midrule
located  &  (0.116)  &  Be  &  (0.202)  &  Be  &  (0.609)  \\ 
\textcolor{blue}{ij}  &  \textcolor{blue}{(0.992)}  &  \textcolor{blue}{ij}  &  \textcolor{blue}{(0.995)}  &  \textcolor{blue}{ij}  &  \textcolor{blue}{(0.996)}  \\ 
\textcolor{blue}{ing}  &  \textcolor{blue}{(0.999)}  &  \textcolor{blue}{ing}  &  \textcolor{blue}{(0.999)}  &  \textcolor{blue}{ing}  &  \textcolor{blue}{(0.999)}  \\ 
,  &  (0.370)  &  .  &  (0.535)  &  .  &  (0.503)  \\
\textbackslash n  &  (0.300)  &  \textbackslash n  &  (0.310)  &  \textbackslash n  &  (0.227)  \\ 
The  &  (0.126)  &  The  &  (0.150)  &  The  &  (0.158)  \\
capital  &  (0.074)  &  capital  &  (0.125)  &  capital  &  (0.102)  \\
of  &  (0.757)  &  \textcolor{blue}{of}  &  \textcolor{blue}{(0.832)}  &  \textcolor{blue}{of}  &  \textcolor{blue}{(0.860)}  \\
China  &  (0.523)  &  the  &  (0.094)  &  the  &  (0.115)  \\
country  &  (0.146)  &  United  &  (0.438)  &  United  &  (0.436)  \\
\textcolor{blue}{States}  &  \textcolor{blue}{(0.880)}  &  \textcolor{blue}{States}  &  \textcolor{blue}{(0.816)}  &  \textcolor{blue}{States}  &  \textcolor{blue}{(0.858)}  \\
is  &  (0.733)  &  is  &  (0.617)  &  is  &  (0.687)  \\
\textcolor{blue}{Washington}  & \textcolor{blue}{(0.877)}  &  \textcolor{blue}{Washington}  &  \textcolor{blue}{(0.992)}  &  \textcolor{blue}{Washington}  &  \textcolor{blue}{(0.989)}  \\
,  &  (0.443)  &  ,  &  (0.394)  &  ,  &  (0.462)  \\
D  &  (0.571)  &  D  &  (0.728)  &  D  &  (0.686)  \\
\textcolor{blue}{.}  &  \textcolor{blue}{(0.999)}  &  \textcolor{blue}{.}  &  \textcolor{blue}{(0.997)}  &  \textcolor{blue}{.}  &  \textcolor{blue}{(0.997)}  \\ 
\textcolor{blue}{C}  &  \textcolor{blue}{(0.969)}  &  \textcolor{blue}{C}  &  \textcolor{blue}{(0.967)}  &  \textcolor{blue}{C}  &  \textcolor{blue}{(0.934)}  \\
\textcolor{blue}{.}  &  \textcolor{blue}{(0.830)}  &  \textcolor{blue}{.}  &  \textcolor{blue}{(0.873)}  &  \textcolor{blue}{.}  &  \textcolor{blue}{(0.904)}  \\ 
\bottomrule
\end{tabular}
\end{table}

\subsection{Training with different early-exit configurations}
\label{subsec:exp_training_losses_more_config}

As mentioned in Section~\ref{sec:overview}, \sys supports a wide variety of early-exit configurations, such as numbers, locations, structures and training loss weights of early exits.
Here, we demonstrate the training of two additional 7B early-exit models.
The first model has the same configurations as those of the 7B early-exit model considered in Section~\ref{subsec:exp_training}, except that each early-exit layer contains layer normalization and an MLP, on top of the output embedding matrix.
The second model has three early exits located at Layer 4, 12 and 20, respectively, all containing layer normalization, with early-exit loss weights set to 0.1, 0.2 and 0.3.
In addition, all embedding matrices are tied.
Our empirical results, as shown in Figure~\ref{fig:training_loss_various_configurations}, confirm the convergence of training losses for both models.
It would be interesting future work to compare more thoroughly different early-exit configurations for both training and inference.

\begin{figure}
\centering
\includegraphics[width=.45\textwidth]{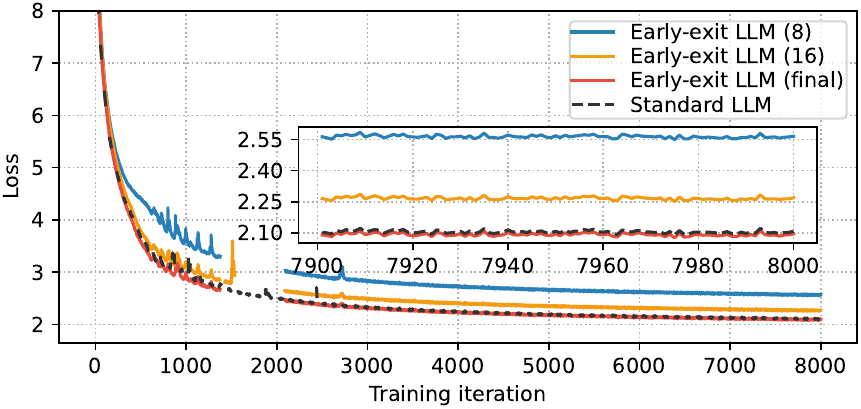}%
\includegraphics[width=.45\textwidth]{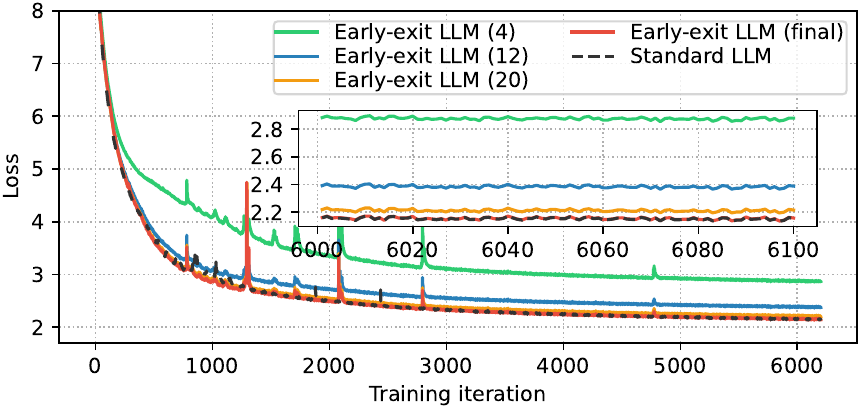}
\caption{Convergence of early-exit/final-exit training losses for two additional models with different configurations.
Each curve is annotated with the index of the Transformer layer that the corresponding exit is connected to.
The missing parts in the curves of the first figure were caused by technical issues with wandb (\url{https://wandb.ai}) when the experiments were running.
}
\label{fig:training_loss_various_configurations}
\end{figure}

\section{Advanced features}
\label{sec:appendix_advanced_features}

This section presents more details about the advanced features introduced in Section~\ref{subsec:additional_features} for training early-exit LLMs.

\subsection{Non-constant early-exit loss weights}
\label{subsec:non_constant_loss_weights}

The weights of early-exit losses in the training objective of Eq.~\eqref{eq:multi_exit_training_loss} can be changing rather than constant during the training process, just like the learning rate or other hyperparameters.
Allowing the weights to change can offer more fine-grained control of how gradients from multiple losses jointly impact the backbone and early/final output layers of the model.

One concrete option that we offer in \sys is \emph{warm-up}.
With this option, early-exit loss weights start at small values, and gradually increase with training iterations until reaching the pre-specified maximum values.
This approach has been adopted in prior works \cite{Kaya2018}.
The idea is to encourage the deep neural network to primarily optimize for the full-model output quality from the beginning of the training process,
while the skill of early exiting is gradually acquired with minor or no negative impact on the final outputs of the full model.

Another option is \emph{cool-down}, which does the opposite and decreases early-exit loss weights during the training process.
This option is inspired by prior works \cite{Lee2014DeeplySupervisedN,Szegedy2014GoingDW} that leverage early-exit losses for the purpose of regularizing the training process of deep neural networks.
Such ``deep supervision'' provided by early-exit losses can stabilize and accelerate the convergence of training, and potentially improve the intermediate features learned by the neural network.
As early-exit loss weights gradually decay, the regularization gets weaker, so that the network becomes more focused on its primary objective, namely the final-output quality.

\subsection{Filling explicit bubbles}
\label{subsec:appendix_fill_pipeline_bubble}

To further leverage the explicit bubbles within the 1F1B pipeline schedule, i.e.~the light gray areas in Figure~\ref{fig:pipeline_fill_bubbles},  
we design and implement a novel approach of filling them with \emph{partial} forward/backward computation of additional microbatches.
See Figure~\ref{fig:pipeline_fill_bubbles} for a visualization.
This is primarily inspired by the idea from \cite{Osawa2022}: instead of designing a new, sophisticated schedule that has a lower bubble ratio, one may seek to fill the bubbles of an existing schedule with useful computation, which leads to better resource utilization and faster training.
Note that our method changes the optimization semantics.
From the perspective of stochastic optimization, we can prove formally that with such additional computation and under certain conditions, one obtains an \emph{unbiased} gradient estimate with \emph{reduced variance} for the original training objective.
The remaining of this section is dedicated to details of the methodology and analysis.

\subsubsection{Methodology}

Our approach is explained below.
For notational simplicity, let us call the explicit bubbles between the warm-up and steady phases as Part 1, and the bubbles during the cool-down phase as Part 2.
For each part, we fill them with some computation for $K$ additional microbatches.
More concretely, for Part 1, the $i$-th inserted microbatch (where $i \in [K]$) goes through forward computation of the first $K+1-i$ pipeline stages, followed by backward computation of all visited early-exit losses;
for Part 2, each inserted microbatch goes through forward computation of all stages, followed by backward computation of the final and early-exit losses (if any) only for the last few stages.
In this way, each training iteration can process more data without any time overhead, as long as the number of inserted microbatches and the number of stages for partial forward/backward computation are chosen appropriately.

\paragraph{How many microbatches can be inserted?}

The maximum number of microbatches that can be inserted into Part 1 or 2 of the explicit bubbles, 
without increasing training time per iteration, is
$\lfloor(p-1) b / (f+b)\rfloor = \lfloor(p-1) / (f/b + 1)\rfloor$,
where $f/b$ is (an estimate of) the ratio between forward and backward time.
To see this for Part 1 (resp.~2), one simply needs to notice that the first (resp.~last) pipeline stage has a bubble size $b (p-1)$, while the total forward and backward time for each microbatch is $f+b$. 
Dividing the first value by the second one concludes the proof.

\paragraph{Details about Part 1.}

With $K$ inserted microbatches, the $i$-th microbatch is supposed to go through the forward pass of the first $K+1-i$ stages. 
However, if there is no early exit on Stage $K+1-i$, then this microbatch only needs to go through the forward pass up to the last stage (among the first $K+1-i$ stages) that has at least one early exit.

\paragraph{Details about Part 2.}

We can calculate the maximum number of backward stages for each inserted microbatch, while ensuring that no overhead to training time per iteration occurs.
Notice that for the $i$-th microbatch, the remaining bubble size after its forward computation at the last stage is $(p-1)b - f - (i-1)(f+b) = p b - i (f+b)$.
Dividing this value by $b$ gives the number of backward stages 
$\lfloor (p b - i (f+b)) / b \rfloor = \lfloor p - i (f/b + 1) \rfloor$.

\paragraph{Some remarks.}

(1) There are some limitations to the proposed approach.
For example, it requires that there is no tied/shared model parameter across pipeline stages; otherwise, the gradients from partial forward/backward passes might (in theory) be harmful rather than beneficial for training the model.
Another concern is that inserting microbatches into Part 1 of the explicit bubbles can cause additional overhead of activation memory in early stages.
(2) While Part 1 of this approach requires the existence of early exits, Part 2 is actually applicable to training standard models without early exits as well.
(3) One might raise concerns about the inefficiency of data usage, as some microbatches only undergo partial forward/backward passes.
In general, this is not an issue for LLM pre-training, since training data is usually abundant, and the complete training process might not even go through one epoch of the data.


\paragraph{A different perspective.}

The actual implementation of this method in \sys takes a different perspective.
Instead of inserting additional microbatches into the original schedule, we keep the number of microbatches per training iteration unchanged, and do the following:
(1)~we replace the full forward/backward computation of a few microbatches with partial forward/backward computation, which can be placed in the bubble between the warm-up and steady phases;
(2)~we truncate the backward computation of the last few microbatches.
These two steps correspond exactly to Parts 1 and 2 introduced earlier, and the visualization in Figure~\ref{fig:pipeline_fill_bubbles}(b) remains valid.
Such an implementation reduces the training time per iteration, at the cost of lower data utilization.


\subsubsection{Analysis}

\begin{figure}[tbp]
\centering
\includegraphics[width=0.5\textwidth]{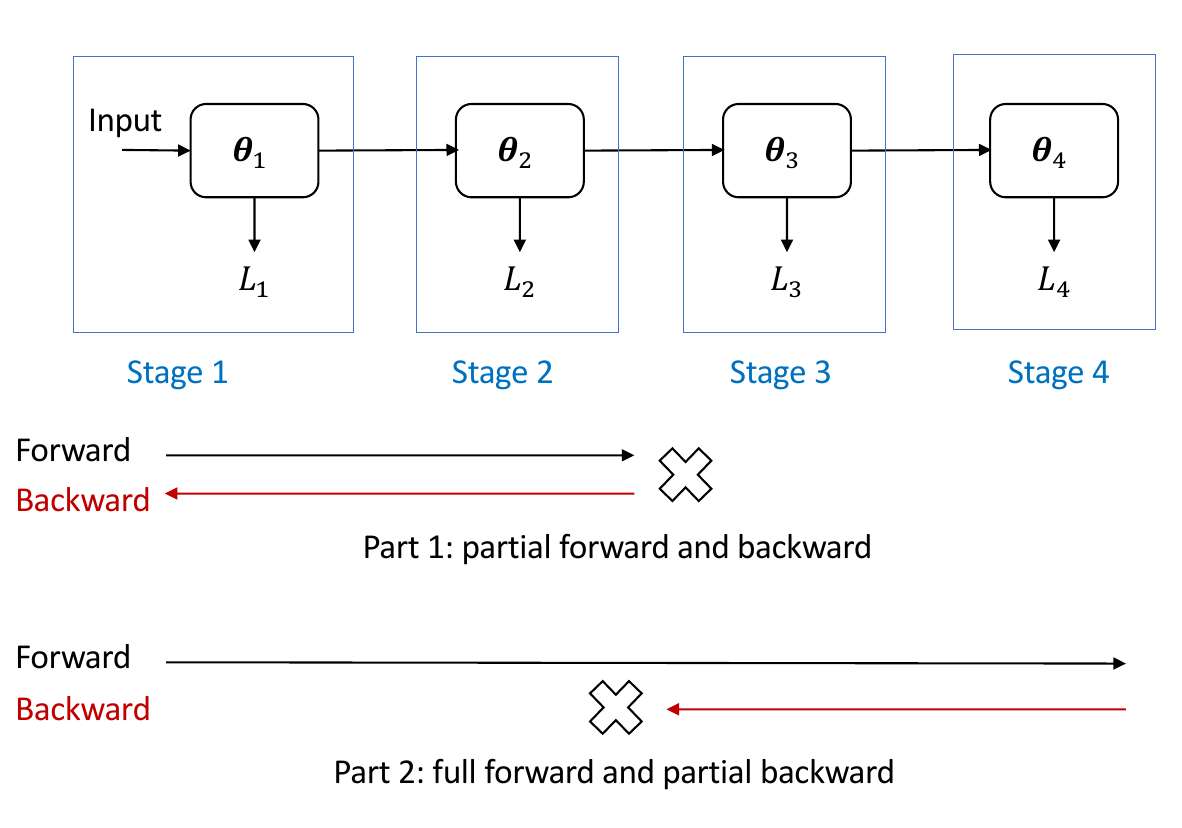}
\caption{An illustration of the partial forward/backward passes when filling pipeline bubbles with additional microbatches.}
\label{fig:pipeline_partial_forward_backward}
\end{figure}

We argue that such partial forward/backward computation offers useful gradient information for training, under the assumption that there is no tied model parameter across pipeline stages.
For a theoretical justification, let us take the perspective of stochastic optimization, and assume that each microbatch is sampled independently from some data distribution.

\begin{claim}[Informal]
Under the above assumptions, the accumulated gradient of one training iteration, with extra updates from additional microbatches inserted into Part 2, remains (after some appropriate entrywise scaling) an \emph{unbiased} estimate, but with \emph{reduced variance}, of the gradient for the targeted population risk.
A similar claim holds true for Part 1 as well, except for certain extreme (and unlikely) situations where gradients of different early/final losses have a strong negative correlation.
\end{claim}

In the remaining of this section, we provide an informal proof for this claim, in a concrete example with four pipeline stages and one additional microbatch inserted into either Part 1 or Part 2;
a visualization can be found in Figure~\ref{fig:pipeline_partial_forward_backward}.
Recall that the training objective is defined as $\loss = \sum_{i \in [4]} \loss_i$, where each $\loss_i$ is a weighted sum of all losses on Stage $i$.
We denote the model parameters as $\btheta = [\btheta_1, \btheta_2, \btheta_3, \btheta_4]$, with correspondence to how the model is partitioned into four pipeline stages.
One key observation that will be useful is that $\partial \loss_i / \partial \btheta_j = \bm{0}$ for any $1 \le i < j \le 4$, due to the sequential nature of the model.

\paragraph{Analysis of Part 2.}

We start with an analysis for Part 2, which is slightly simpler.
Suppose that an additional microbatch goes through the full forward pass, followed by a partial backward pass covering only the last two stages, namely $\btheta_4$ and $\btheta_3$.
Then, the gradient from this microbatch is
\begin{equation*}
\text{gradient} = \bigg[\bm{0}, \bm{0}, \frac{\partial (\loss_3 + \loss_4)}{\partial \btheta_3}, \frac{\partial (\loss_3 + \loss_4)}{\partial \btheta_4}\bigg]
= \bigg[\bm{0}, \bm{0}, \frac{\partial \loss}{\partial \btheta_3}, \frac{\partial \loss}{\partial \btheta_4}\bigg];
\end{equation*}
here, the first equality is due to Proposition~\ref{prop:auxiliary_loss_bp} and the partial backward pass, while the second equality follows from the observation that $\partial \loss_i / \partial \btheta_j = \bm{0}$ for any $i < j$.

Now, suppose that a total of $B$ microbatches was originally used for one training iteration, and we let $\loss^{(k)}$ denote the loss of the $k$-th microbatch.
Then, with the additional $(B+1)$-th microbatch inserted into Part 2 of the explicit bubbles, the accumulated gradient of one training iteration (normalized by $1/B$) becomes:
\begin{align*}
\text{accumulated gradient} 
&= \frac{1}{B} \Bigg( \sum_{k \in [B]} \frac{\partial \loss^{(k)}}{\partial \btheta} + \bigg[\bm{0}, \bm{0}, \frac{\partial \loss^{(B+1)}}{\partial \btheta_3}, \frac{\partial \loss^{(B+1)}}{\partial \btheta_4}\bigg] \Bigg) \\
&= \bigg[ 
\frac{1}{B}\sum_{k \in [B]} \frac{\partial \loss^{(k)}}{\partial \btheta_1}, 
\frac{1}{B}\sum_{k \in [B]} \frac{\partial \loss^{(k)}}{\partial \btheta_2},
\frac{1}{B}\sum_{k \in [{\color{red}B+1}]} \frac{\partial \loss^{(k)}}{\partial \btheta_3},
\frac{1}{B}\sum_{k \in [{\color{red}B+1}]} \frac{\partial \loss^{(k)}}{\partial \btheta_4} \bigg]
\end{align*}
Taking the expectation over the data distribution, we realize that the additional microbatch essentially scales up gradients for $\btheta_3$ and $\btheta_4$ by a factor of $(B+1) / B$.
Or, with a simple entrywise scaling on the gradients for $\btheta_3$ and $\btheta_4$ by a factor of $B / (B+1)$, we recover an unbiased estimate of the gradient for the targeted population risk, with reduced variance:
\begin{align*}
\text{accumulated gradient} 
&= \bigg[ 
\frac{1}{B}\sum_{k \in [B]} \frac{\partial \loss^{(k)}}{\partial \btheta_1}, 
\frac{1}{B}\sum_{k \in [B]} \frac{\partial \loss^{(k)}}{\partial \btheta_2},
\frac{1}{{\color{red}B+1}}\sum_{k \in [{\color{red}B+1}]} \frac{\partial \loss^{(k)}}{\partial \btheta_3},
\frac{1}{{\color{red}B+1}}\sum_{k \in [{\color{red}B+1}]} \frac{\partial \loss^{(k)}}{\partial \btheta_4} \bigg].
\end{align*}

\paragraph{Analysis of Part 1.}

Suppose that an additional microbatch goes through the forward and backward computation only for the first two stages, corresponding to model parameters $\btheta_1$ and $\btheta_2$.
Then, the gradient from this microbatch is
\begin{equation*}
\text{gradient} = \bigg[\frac{\partial (\loss_1 + \loss_2)}{\partial \btheta_1}, \frac{\partial (\loss_1 + \loss_2)}{\partial \btheta_2}, \bm{0}, \bm{0}\bigg]
= \frac{\partial (\loss_1 + \loss_2)}{\partial \btheta};
\end{equation*}
here, the first equality is due to the partial forward and backward passes, and the second equality follows again from the observation that $\partial \loss_i / \partial \btheta_j = \bm{0}$ for any $i < j$.

Now, suppose that a total of $B$ microbatches was originally used for one training iteration.
We denote the loss of the $k$-th microbatch as $\loss^{(k)} = \sum_{i \in [4]} \loss^{(k)}_i$.
Then, with the additional $(B+1)$-th microbatch inserted into Part 1 of the explicit bubbles, the accumulated gradient of one training iteration (normalized by $1/B$) becomes:
\begin{align*}
\text{accumulated gradient} &= \frac{1}{B} \bigg(\sum_{k \in [B]} \frac{\partial \loss^{(k)}}{\partial \btheta} + \frac{\partial (\loss^{(B+1)}_1 + \loss^{(B+1)}_2)}{\partial \btheta} \bigg) \\
&= \frac{1}{B} \bigg(\sum_{k \in [B]} \frac{\partial (\loss^{(k)}_1 + \loss^{(k)}_2 + \loss^{(k)}_3 + \loss^{(k)}_4)}{\partial \btheta} + \frac{\partial (\loss^{(B+1)}_1 + \loss^{(B+1)}_2)}{\partial \btheta} \bigg) \\
&= \frac{1}{B} \bigg(\sum_{k \in [{\color{red}B+1}]} \frac{\partial (\loss^{(k)}_1 + \loss^{(k)}_2)}{\partial \btheta} + \sum_{k \in [B]} \frac{\partial (\loss^{(k)}_3 + \loss^{(k)}_4)}{\partial \btheta} \bigg).
\end{align*}
Taking the expectation over the data distribution, we see that the additional microbatch essentially scales up the weights of $\loss_1$ and $\loss_2$ in the gradient by a factor of $(B+1) / B$.
Or, if the weights of $\loss_1$ and $\loss_2$ are manually scaled by a factor of $B / (B+1)$ during training, then we recover an unbiased gradient estimate for the original risk $\loss = \sum_{i \in [4]} \loss_i$:
\begin{equation*}
\text{accumulated gradient} = \frac{1}{{\color{red}B+1}} \sum_{k \in [{\color{red}B+1}]} \frac{\partial (\loss^{(k)}_1 + \loss^{(k)}_2)}{\partial \btheta} + \frac{1}{B} \sum_{k \in [B]} \frac{\partial (\loss^{(k)}_3 + \loss^{(k)}_4)}{\partial \btheta} 
\end{equation*}
We claim that this leads to reduced gradient variance, \emph{unless} gradients of early and later losses have a strong negative correlation.
In the formal analysis of variance below, we consider scalars for notational simplicity, though it can be easily generalized to vectors or tensors by replacing scalar multiplication with inner product.
\begin{proposition}
\label{prop:variance}
Consider two distributions $\mathcal{A}$ and $\mathcal{B}$ with means $\astar$ and $\bstar$, respectively.
Let $a_1, \dots, a_N, a_{N+1}$ be independent and identically distributed (i.i.d.) samples of $\mathcal{A}$, and $b_1, \dots, b_N$ be i.i.d.~samples of $\mathcal{B}$, but for each $i$, $a_i$ can be correlated with $b_i$.
Consider two estimates of $\astar + \bstar$, defined as follows:
\begin{equation*}
    \ehat \coloneqq \ahat_N + \bhat_N, \quad \ehatp \coloneqq \ahat_{N+1} + \bhat_N, \quad \text{where}
\end{equation*}
\begin{equation*}
    \ahat_{k} \coloneqq \frac{1}{k}\sum_{i \in [k]} a_i, \quad 
    \bhat_{k} \coloneqq \frac{1}{k} \sum_{i \in [k]} b_i, \quad k \in \{N, N+1\}.
\end{equation*}
Then it holds that 
\begin{equation*}
    \E[\ehat] = \E[\ehatp] = \astar + \bstar, \quad 
    \var(\ehat) - \var(\ehatp) = \frac{1}{N(N+1)} \var(a_1) + \frac{2}{N(N+1)} \cov(a_1, b_1).
\end{equation*}
\end{proposition}
In this proposition, $\ahat_N, \ahat_{N+1}$ and $\bhat_N$ correspond to 
$\frac{1}{{B}} \sum_{k \in [{B}]} \frac{\partial (\loss^{(k)}_1 + \loss^{(k)}_2)}{\partial \btheta}$,
$\frac{1}{{B+1}} \sum_{k \in [{B+1}]} \frac{\partial (\loss^{(k)}_1 + \loss^{(k)}_2)}{\partial \btheta}$ and
$\frac{1}{B} \sum_{k \in [B]} \frac{\partial (\loss^{(k)}_3 + \loss^{(k)}_4)}{\partial \btheta}$ 
in our previous analysis of accumulated gradients, respectively.
Hence our claim on gradient variance follows immediately from the conclusion of this proposition.
\begin{proof}[Proof of Proposition~\ref{prop:variance}]
First, unbiasedness is obvious.
As for variance, we have the following elementary calculation:
\begin{align*}
    \var(\ehat) &= \E\big[(\ahat_N + \bhat_N - \astar - \bstar)^2\big] = \E\big[(\ahat_N - \astar)^2\big] + \E\big[(\bhat_N - \bstar)^2\big] + 2 \E\big[(\ahat_N - \astar)(\bhat_N - \bstar)\big], \\
    \var(\ehatp) &= \E\big[(\ahat_{N+1} + \bhat_N - \astar - \bstar)^2\big] = \E\big[(\ahat_{N+1} - \astar)^2\big] + \E\big[(\bhat_N - \bstar)^2\big] + 2 \E\big[(\ahat_{N+1} - \astar)(\bhat_N - \bstar)\big].
\end{align*}
Moreover, 
\begin{equation*}
    \E\big[(\ahat_N - \astar)^2\big] = \frac{1}{N} \var(a_1), \quad \E\big[(\ahat_{N+1} - \astar)^2\big] = \frac{1}{N+1} \var(a_1), 
\end{equation*}
and 
\begin{align*}
    \E\big[(\ahat_N - \astar)(\bhat_N - \bstar)\big] &= \frac{1}{N^2} \sum_{i\in[N]} \E\big[(a_i - \astar)(b_i - \bstar)\big] = \frac{1}{N} \cov(a_1, b_1), \\
    \E\big[(\ahat_{N+1} - \astar)(\bhat_N - \bstar)\big] &= \frac{1}{N(N+1)} \sum_{i\in[N]} \E\big[(a_i - \astar)(b_i - \bstar)\big] = \frac{1}{N+1} \cov(a_1, b_1).
\end{align*}
Putting things together,
\begin{align*}
    \var(\ehat) - \var(\ehatp) &= \E\big[(\ahat_N - \astar)^2\big] - \E\big[(\ahat_{N+1} - \astar)^2\big] \\
    &\quad\quad + 2 \Big(\E\big[(\ahat_N - \astar)(\bhat_N - \bstar)\big] - \E\big[(\ahat_{N+1} - \astar)(\bhat_N - \bstar)\big]\Big) \\
    &= \frac{1}{N(N+1)} \var(a_1) + \frac{2}{N(N+1)} \cov(a_1, b_1),
\end{align*}
which concludes our proof.
\end{proof}

\section{Implementations}
\label{sec:implementation}

The implementation of \sys is based on Megatron-LM \cite{Narayanan2021}, primarily extending Megatron-LM's model architectures, pipeline scheduling, and inference service to support the training and inference of early-exit LLMs.
We introduce each of these aspects in more detail below.

\subsection{Model architectures}

We have introduced a new class of models called \texttt{EarlyExitGPTModel}, which is the early-exit counterpart of \texttt{GPTModel} in the original model library of Megatron-LM.
The model is constructed with a few other classes, including \texttt{EarlyExitTransformerLayer}, \texttt{EarlyExitTransformer}, and \texttt{EarlyExitLanguageModel}.
\texttt{EarlyExitTransformerLayer} is a replacement for the original \texttt{ParallelTransformerLayer} in Megatron-LM. 
It adds an early-exit structure on top of the standard Transformer layer, which allows it to generate outputs for both the main network backbone and the early exit; for the latter, it returns a lazy loss function during training, or tokens during inference.
This module supports various customizations of the early-exit structure; 
besides the minimalistic structure with an output embedding matrix and an optional output normalization layer, 
one might add e.g.~a MLP or a complete Transformer layer. 
These additional structures can be combined in any desired manner and can be placed before or after the backbone part of this layer.
On the other hand, \texttt{EarlyExitTransformer} and \texttt{EarlyExitLanguageModel} are mainly used to propagate the early-exit outputs to the top-level model. 
They are capable of stopping the forward computation at the early-exit layer and returning the intermediate outputs, which facilitates accelerated inference.

\subsection{Pipeline scheduling}

We have adjusted the existing 1F1B schedule for early-exit LLMs, as shown in Figure~\ref{fig:pipeline_schedules}.
To fill implicit bubbles and reduce GPU memory overhead, lazy loss functions of early-exit layers are returned together with outputs of the backbone network during forward steps. 
These lazy functions are not actually called until their corresponding auxiliary losses (cf.~Section~\ref{subsubsec:Methodology}) are calculated in the backward steps.
For the method of filling explicit bubbles proposed in Section~\ref{subsec:additional_features}, we have inserted partial forward/backward computation of additional microbatches into warm-up and cool-down phases of the 1F1B schedule.
The number of inserted microbatches and partial forward/backward stages can be automatically calculated through the user-specified (estimate of) ratio between backward and forward time.

\subsection{Inference service}
\label{subsec:implementation_inference}

To support inference of early-exit LLMs, we have refactored the text-generation module of Megatron-LM.

For inference with pipeline parallelism, i.e.~the \emph{pipeline-based} approach proposed in Section~\ref{sec:inference}, we have re-implemented the forward process.
With our implementation, the first pipeline stage will wait for an exit signal from the early/final exits of all subsequent stages after its forward computation is completed.
Each subsequent stage will send an exit signal and the output token to the first stage, if there is an exit within the stage that satisfies the exit condition.
Upon receiving the signal and generated token, the first stage will immediately start the forward pass for generating the next token.
With this implementation, regardless of the early-exit layers' positions in subsequent stages, the inference service can immediately generate a token whenever early exiting happens on some stage, without waiting for the completion of the entire stage (except for the first stage).

For inference without pipeline parallelism, we have implemented a mechanism of \emph{KV recomputation}, which is a variant of synchronized parallel decoding proposed recently in \cite{Bae2023}. 
In this approach, we maintain a list of the most recent tokens that have missing KV caches in deep layers due to early exiting. 
During each forward pass, we include these early-exit tokens in the current forward pass, which allows for direct recomputation of the KV caches for these tokens and thus avoids the issue of missing KV caches.
Acceleration of sequence generation is still achieved, thanks to the batching effects of GPU computation.
To avoid the endless accumulation of early-exit tokens, we enforce a full-model forward pass whenever the number of early-exit tokens reaches a pre-specified value.

\section{Preliminaries}
\label{sec:preliminaries}

\paragraph{Transformers.}

The Transformer architecture \cite{vaswani2017attention,tay2022efficient} has been playing a dominant role in natural language processing (NLP) and large language models (LLMs) \cite{Bommasani2021OnTO,Zhao2023ASO}.
It is typically composed of an input embedding layer, a stack of Transformer layers, and finally an output layer.
Each Transformer layer consists of cross-attention and/or self-attention modules \cite{BahdanauCB14,kim2017structured,parikh2016decomposable}, a multi-layer perceptron (MLP), and layer normalization (LayerNorm \cite{ba2016layernorm} or RMSNorm \cite{zhang2019root}).
Transformers can be categorized into three types: encoder-only, encoder-decoder, and decoder-only.
For the latter two, there is an output embedding matrix in the output layer, which transforms hidden states into logits on a (typically large) vocabulary that can be used for generating tokens.
An LLM can be learned by unsupervised pre-training, e.g.~minimizing the negative log-likelihood of next-token prediction on a large corpus \cite{Radford2018ImprovingLU,Radford2019LanguageMA}.
In this work, we focus on the decoder-only generative pre-training (GPT) Transformer architecture \cite{Radford2018ImprovingLU,Radford2019LanguageMA}, though many of our ideas are widely applicable to other Transformer architectures or generic deep neural networks.

\paragraph{Early-exit LLMs.}

An early-exit neural network can be obtained by adding to a standard neural network some early-exit layers that turn intermediate hidden states into early outputs \cite{Xin2020,Schwartz2020}.
During inference for a given input, the model starts a forward pass and decides (at each early exit) whether to return an output or continue forwarding via certain rules, e.g.~to return an output whenever the confidence of prediction is above a pre-defined threshold \cite{Schwartz2020,Schuster2022}.


The standard way of training an early-exit model is to minimize a weighted sum of early-exit and final-exit training losses \cite{Schwartz2020,Schuster2022}.
Note that early-exit layers bring additional computational overhead to training.
This is especially the case for LLMs, primarily due to the large output embedding matrix of size $h \times V$ within each early-exit layer, where $h$ is the hidden dimension and $V$ is the vocabulary size.
We call an early-exit layer \emph{minimalistic} if it has the same structure as the final output layer of the GPT model architecture \cite{Radford2018ImprovingLU,Radford2019LanguageMA}, which includes an output embedding matrix, plus an optional LayerNorm/RMSNorm in front of it.
Additional modules can be added to early-exit layers for increased expressivity and adaptivity of early exits.

\paragraph{3D parallelism.}

3D parallelism refers to the combination of data, tensor, sequence and pipeline parallelism,
and has been implemented in state-of-the-art LLM frameworks such as 
Megatron-LM \cite{Shoeybi2019,Narayanan2021,Korthikanti2022,Smith2022},
DeepSpeed \cite{Rasley2020DeepSpeedSO,Smith2022}, 
Mesh-TensorFlow \cite{Shazeer2018Mesh}, 
Alpa \cite{Zheng2022Alpa}, 
InternLM \cite{2023internlm}, 
among others.
With \emph{data parallelism}, each GPU handles the forward and backward computation for one part of the data batch, and then the results are aggregated at the end of the training iteration.
When the model is too large to fit in a single GPU, model partitioning becomes necessary and can be used in conjunction with data parallelism.
With \emph{tensor (and sequence) parallelism}, each large module (e.g.~a linear layer) is divided into multiple pieces that are assigned to different GPUs, so that each computational task related to it (e.g.~large matrix multiplication) can be divided into smaller tasks and solved in parallel.
One major limitation of tensor (and sequence) parallelism is that it requires expensive collective communication such as all-reduce operations, and thus is only viable for high-end GPUs within the same computing node, with high-bandwidth communication among them.

\emph{Pipeline parallelism} \cite{Narayanan2019,Narayanan2021-PipeDream,Fan2021,Li2021}, on the other hand, partitions a deep model along the depth dimension into multiple pipeline stages.
Moreover, each data batch is divided into multiple microbatches, and their forward/backward computational tasks are scheduled among those multiple pipeline stages.
More specifically, each stage performs the forward computation for each microbatch and sends the resulting hidden states to another stage; later on, it performs the backward computation for the same microbatch after receiving the gradients of the training objective with respect to the sent hidden states.
Pipeline parallelism only requires sparse and inexpensive point-to-point (P2P) communication between pipeline stages, which makes it applicable and oftentimes must-have in much broader scenarios when tensor (and sequence) parallelism is infeasible or insufficient, whether in GPU clusters or in decentralized settings \cite{Yuan2022,yang2023holmes}.
The main concern with pipeline parallelism is its low utilization rate of computational resources, due to pipeline bubbles and load imbalance across pipeline stages \cite{Narayanan2021}. 

\section{Related works}
\label{sec:related_works}

\paragraph{Early exiting.}

As introduced in Section~\ref{sec:introduction}, early exiting has been widely applied for accelerating inference of deep neural networks in the literature \cite{Graves2016AdaptiveCT,Liu2020,Hou2020,Zhou2020,Elbayad2020,Schwartz2020,Xin2020,Schuster2021,Xin2021,Li2021AcceleratingBI,Hu2023SmartBERTAP,Teerapittayanon2016,Huang2018,Kaya2018,Laskaridis2021AdaptiveIT,Scardapane2020WhySW,Han2021DynamicNN,Xu2023SurveyDynamic,Dai2023ApparateRE}.
In terms of early-exit Transformers in the NLP domain, the majority of prior works are focused on BERT \cite{Devlin2019BERTPO} or other encoder-only models for classification tasks \cite{Liu2020,Hou2020,Zhou2020,Elbayad2020,Schwartz2020,Xin2020,Schuster2021,Xin2021,Li2021AcceleratingBI,Hu2023SmartBERTAP}.
Recent works have begun to study token-wise early exiting for accelerating inference of encoder-decoder or decoder-only LLMs in autoregressive sequence generation \cite{Schuster2021,DelCorro2023,Bae2023,Varshney2023AcceleratingLI}.
While they are largely focused on designing \emph{inference} mechanisms, 
the lack of support in prior works for \emph{training} early-exit models with massive 3D parallelism inevitably poses an obstacle to truly scaling up early-exit LLMs.
Our work on \sys is one important step towards eliminating this obstacle, and
we also contribute a novel pipeline-based inference mechanism along the way.


\paragraph{Early-exit inference with KV caching.}

Several approaches have been recently proposed to resolve the conflict between early exiting and KV caching in autoregressive generation.
One approach \cite{Elbayad2020,Li2021AcceleratingBI,Schuster2022} is to copy the hidden states of the current token at the exiting layer to all later layers, which will be used to compute the keys and values at later attention layers. 
Despite its efficiency, this method causes a deviation from the inference process that the model is trained to excel at, which can harm the output quality.
Another solution \cite{DelCorro2023} is to pre-specify the exiting layer for each token, while ensuring that KV missing in previous tokens will not hinder generation of later tokens;
with this approach, the ability of token-wise adaptive selection of exits is inevitably lost.
The third method \cite{Bae2023,tang2023deed} is to store the hidden states of recent tokens that were generated with early exiting, and whenever KV missing happens, run a batch forward pass with the current and recent tokens to fulfill the KV caches.
Acceleration by this method relies on the batching effect of GPU computation.
In comparison, our proposed pipeline-based method achieves early-exit acceleration in theoretical time complexity.

\paragraph{Other applications of early-exit models.}

Interestingly, early exiting has been used for many other purposes in the literature, besides accelerating inference.
For example, early-exit losses can provide deep supervision and regularization for training deep networks with enhanced stability of convergence \cite{Szegedy2014GoingDW,Lee2014DeeplySupervisedN}, 
which motivates our implementation of the cool-down option for non-constant early-exit loss weights,
as explained in Appendix~\ref{subsec:non_constant_loss_weights}.
Other works use early exits for interpreting the intermediate features learned by each layer of a Transformer \cite{Langedijk2023DecoderLensLI}, 
for improving outputs of the final exit \cite{Gera2023TheBO},
or as the draft models in speculative decoding \cite{Kim2023SpeculativeDW}.
Early exiting has also found applications in
long-tailed classification \cite{Duggal2020ELFAE},
federated learning with heterogeneous clients \cite{Ilhan2023ScaleFLRF,Zhong2022SemiHFLSF},
token reduction \cite{liu2024a},
layer-wise training of deep networks \cite{Scardapane2020WhySW},
among others.
We believe that our progress in this work can be beneficial for these various purposes.


\paragraph{Other methods of accelerating LLM inference.}

It is worth mentioning that there are other types of dynamic neural networks \cite{Han2021DynamicNN,Xu2023SurveyDynamic} that facilitate conditional computation and elastic inference, 
such as layer skipping \cite{Bengio2013EstimatingOP,Bengio2015ConditionalCI,Wang2022SkipBERTEI,DelCorro2023,YomDin2023JumpTC,Zeng2023LearningTS,Wang2023HadSkipHA}
and mixtures of experts \cite{Jacobs1991AdaptiveMO,Shazeer2017OutrageouslyLN,Fedus2021SwitchTS}.
Another line of work aims to accelerate LLM inference 
by designing new decoding methods rather than model architectures \cite{Leviathan2022FastIF,Santilli2023AcceleratingTI,zhao2023lookahead}.
Each of these approaches has its own pros and cons, and some of them can be complemented by early exiting.
A detailed comparison between these methods is beyond the scope of the current work.


\end{document}